\documentclass[conference]{IEEEtran}
\IEEEoverridecommandlockouts

\usepackage{cite}
\usepackage{algorithmic}
\usepackage[ruled]{algorithm2e}
\usepackage{textcomp}
\usepackage{xcolor}
\def\BibTeX{{\rm B\kern-.05em{\sc i\kern-.025em b}\kern-.08em
    T\kern-.1667em\lower.7ex\hbox{E}\kern-.125emX}}

\usepackage{lipsum, graphicx}
\usepackage{amssymb,amsmath,amsfonts}
\usepackage{url}
\usepackage{booktabs}

\usepackage{bm}
\usepackage{mathtools}
\usepackage{adjustbox}
\usepackage{tabularx}
\usepackage{multirow}
\usepackage{makecell}
\usepackage{hhline}
\usepackage{tikz}
\usepackage{colortbl}
\usepackage{alphalph}
\usepackage{caption}
\usepackage{subcaption}
\usepackage{amsthm}
\usepackage{siunitx}
\usepackage{cleveref}
\usepackage{enumitem}

\usetikzlibrary{arrows.meta}
\usetikzlibrary{positioning}
\usetikzlibrary{calc}

\def\x{{\mathbf x}}

\def\cvg{{\text{cvg}}}
\def\w{{\mathbf w}}
\def\a{{\mathbf a}}
\def\v{{\mathbf v}}

\def\b{{\mathbf b}}

\def\A{{\mathbf A}}

\def\N{{\mathcal N}}

\def\Real{{\mathbb R}}

\DeclareMathOperator*{\argmin}{arg\,min}

\DeclareMathOperator{\ind}{I}
\DeclareMathOperator{\sgn}{sgn}

\newtheorem{proposition}{Proposition}
\newtheorem{lemma}{Lemma}

\makeatletter
\newcommand{\removelatexerror}{\let\@latex@error\@gobble}
\makeatother

\graphicspath{{figures/}{../}}

\begin{document}

\title{A Framework for Neural Network Pruning Using Gibbs Distributions
	\\
\thanks{The authors would like to thank Fujitsu Laboratories Ltd. and Fujitsu Consulting (Canada) Inc. for providing financial support for this project at the University of Toronto.}
}

\author{\IEEEauthorblockN{Alex Labach}
\IEEEauthorblockA{\textit{Department of Electrical and Computer Engineering} \\
\textit{University of Toronto}\\
Toronto, Canada\\
alex.labach@mail.utoronto.ca}
\and
\IEEEauthorblockN{Shahrokh Valaee}
\IEEEauthorblockA{\textit{Department of Electrical and Computer Engineering} \\
\textit{University of Toronto}\\
Toronto, Canada\\
valaee@ece.utoronto.ca}
}

\maketitle

\begin{abstract}
	Modern deep neural networks are often too large to use in many practical scenarios. Neural network pruning is an important technique for reducing the size of such models and accelerating inference. Gibbs pruning is a novel framework for expressing and designing neural network pruning methods. Combining approaches from statistical physics and stochastic regularization methods, it can train and prune a network simultaneously in such a way that the learned weights and pruning mask are well-adapted for each other. It can be used for structured or unstructured pruning and we propose a number of specific methods for each. We compare our proposed methods to a number of contemporary neural network pruning methods and find that Gibbs pruning outperforms them. In particular, we achieve a new state-of-the-art result for pruning ResNet-56 with the CIFAR-10 dataset.
\end{abstract}


\section{Introduction}
\label{sec:introduction}

Over the past decade, deep learning models have shown remarkable success in many types of data analysis, including image, speech, and text processing. But while they are powerful, state-of-the-art deep neural networks have grown massively in complexity. Using current models on resource-limited devices like smartphones, routers, other edge devices, and embedded devices is often impossible. Even running them at scale on more powerful computers can be prohibitively costly.

Neural network pruning is one way to alleviate this problem. Pruning techniques remove connections or weights within neural networks in such a way that inference accuracy is minimally affected. In many deep neural networks, the majority of connections can be removed without significantly reducing accuracy, which can speed up inference and reduce memory requirements considerably.

Neural network pruning methods can be divided into structured and unstructured pruning. In unstructured pruning, any set of connections can be removed. Structured pruning instead requires that the pruned connections follow some structure, such as removing entire rows from weight matrices or entire filters from convolutional layers. Structured pruning can usually improve performance metrics like inference speed at lower sparsity levels than unstructured pruning, since the structure can be taken advantage of in optimizing performance. On the other hand, unstructured pruning can generally remove more connections without significantly decreasing accuracy.

This paper introduces Gibbs pruning, a family of neural network pruning methods for structured and unstructured pruning that take inspiration from statistical physics.
We apply the concept of Gibbs distributions, which describe stochastic models that take the form of random fields. Gibbs distributions are highly flexible in terms of network properties that they can express, and quadratic energy functions, such as Ising models, can capture parameter interactions to induce structures required for structured pruning. Past work in machine learning and signal processing has applied Gibbs distribtions or similar concepts for parameterizing models~\cite{ackley1985learning}, image restoration~\cite{geman1984stochastic}, and optimization~\cite{kirkpatrick1983optimization}. To the best of our knowledge, Gibbs pruning is the first application of Gibbs distributions for neural network pruning.

Gibbs pruning uses an approach related to stochastic regularization methods like dropout~\cite{hinton2012improving,labach2019survey} in adding random behaviour during training to push the network towards a desirable representation. We induce a Gibbs distribution over the weights of a neural network and sample from it during training to determine pruning masks. This procedure leads to a learned network structure that is resilient to high degrees of pruning. We also take advantage of the temperature parameter in the Gibbs distribution to anneal the distribution, gradually converging to a final pruning mask during training that is well-adapted to the network.

The original contributions to neural network pruning research in this paper are as follows:

\begin{itemize}
	\item A versatile framework for expressing, developing, and understanding neural network pruning methods that connects them to stochastic regularization methods and concepts from statistical physics.
	\item A set of novel neural network pruning methods for both unstructured and structured pruning that outperform existing methods. We establish a new state-of-the-art result for pruning ResNet-56 with the CIFAR-10 dataset.
\end{itemize}

An earlier version of this work has been released as a preprint~\cite{labach2020framework} and has been accepted for publication in the 2020 IEEE Global Communications Conference (GLOBECOM). Changes from the earlier work include proposing more possible Hamiltonian functions, discussing their mathematical properties in more detail, and comparing their performance, as well as providing an extended literature review and more detailed analysis of our results.

This paper is organized as follows. Section \ref{chap:related} is a survey of relevant related works in the area of neural network pruning. Section \ref{chap:theory} gives a formal description of the Gibbs pruning framework and explores various possible pruning methods that can defined within this framework from a mathematical perspective. Section \ref{chap:evaluation} presents the results of various experiments carried out to explore the design space of Gibbs pruning methods and to compare them to existing pruning methods. Finally, Section \ref{chap:conclusion} summarizes our main findings and suggests possible future research directions.

\section{Related Work}
\label{chap:related}

A wide variety of methods for accelerating neural networks by modifying models have been proposed. Broadly speaking, they fall into the following categories.

\begin{itemize}
	\item \textbf{Network pruning}: The removal of connections during or after training.
	\item \textbf{Small models}: Neural network architectures that have been manually designed to approach the performance of large networks while having many fewer parameters, e.g., \cite{howard2017mobilenets, iandola2016squeezenet}.
	\item \textbf{Weight quantization}: Limiting the set of values that weights can take to reduce their memory usage and/or speed up mathematical operations, e.g., \cite{hubara2017quantized,han2016deep}.
	\item \textbf{Tensor decomposition}: Approximating linear transformations within neural networks as products of smaller tensors, e.g., \cite{denton2014exploiting,denil2013predicting}.
	\item \textbf{Knowledge distillation}: Training a small neural network to approximate the function implemented by a large, trained neural network, e.g., \cite{hinton2015distilling,romero2014fitnets}.
\end{itemize}

In practice, multiple methods can be used together to further improve performance. For instance, in \cite{han2016deep}, the authors propose using network pruning alongside weight quantization and Huffman coding to massively reduce the storage requirements of a neural network.

Neural network pruning research dates back to at least the 1980s~\cite{sietsma1988neural}. Since 2015, there has been a resurgence of neural network pruning research focused on developing methods for modern DNNs, which is the line of research that this paper falls into. In this section, we classify existing pruning methods in a number of ways to explain how our proposed method relates to contemporary research.

\subsection{Structured vs. Unstructured Pruning}

Unstructured pruning methods \cite{frankle2019stabilizing,frankle2019lottery,zhu2017prune,guo2016dynamic,han2015learning,lecun1990optimal,dong2017learning,molchanov2017variational,louizos2018learning} are generally usable for any neural network topology, since they are not limited in terms of which connections they prune. Structured pruning methods \cite{li2017pruning,luo2017thinet,molchanov2016pruning,he2018soft,he2019filter,he2017channel,hu2016network,wen2016learning,liu2017learning,ye2018rethinking,yu2018nisp,zhuang2018discrimination,hacene2019attention} are more specific to particular topologies, with nearly all recent methods being designed for CNNs. The most common structures to prune are entire convolutional filters or channels, effectively reducing the size of one tensor dimension for both weight and activation tensors as previously discussed. However, other structures are possible, including pruning entire convolutional kernels (e.g., \cite{molchanov2016pruning}) or pruning all but one connection in a kernel, allowing the convolution operation to be replaced by an operation that simply scales and shifts the input~\cite{hacene2019attention}.

While some existing pruning methods can be used for both structured and unstructured pruning \cite{gomez2019learning,he2018amc}, the vast majority are specialized for one or the other task. Our proposed method defines a framework that can be used for unstructured or structured pruning with arbitrary structures within a layer, making it much more versatile than most existing pruning methods.

\subsection{Post-training vs. In-Training Pruning}

Most pruning methods take a trained neural network as an input, prune some connections, and then \textbf{fine-tune} the network: repeating the training procedure for a number of epochs to recover the lost performance \cite{frankle2019lottery,frankle2019stabilizing,guo2016dynamic,han2015learning,lecun1990optimal,li2017pruning,luo2017thinet,molchanov2016pruning,dong2017learning,he2018amc,he2017channel,hu2016network,yu2018nisp,zhuang2018discrimination}. Pruning and fine-tuning may be repeated multiple times. In contrast, some pruning methods start with an untrained network and perform pruning or adapt the network for pruning during the training procedure \cite{zhu2017prune,gomez2019learning,wen2016learning,he2018soft,he2019filter,liu2017learning,ye2018rethinking,hacene2019attention,molchanov2017variational,louizos2018learning}. We call these approaches \textbf{post-training} and \textbf{in-training} pruning respectively. When a trained network is available, post-training pruning methods usually have an advantage in that fine-tuning is faster than training from scratch would be. But when a trained network is not available, in-training pruning methods will usually be faster since they do not add additional training epochs for fine-tuning. Our proposed method uses in-training pruning, since we believe that this approach can help the training procedure arrive at a final representation that is more robust to pruning than an ordinary trained network would be.

\subsection{Pruning Heuristics}

A central part of neural network pruning research has been developing metrics for deciding how important a connection is to a network's performance. The simplest and most common heuristic is weight magnitude: connections with weights close to zero are pruned \cite{frankle2019lottery,frankle2019stabilizing,zhu2017prune,gomez2019learning,guo2016dynamic,han2015learning,li2017pruning,wen2016learning,he2018soft,he2018amc}. More advanced heuristics have also been developed, such as using the Hessian of the loss with respect to the weights \cite{lecun1990optimal,dong2017learning}, a Taylor expansion approximation of the loss with respect to the weights~\cite{molchanov2016pruning}, and measuring the redundancy of parts of the network~\cite{he2019filter}. Other pruning methods have used sample inputs from the training set to estimate the importance of parameters~\cite{luo2017thinet,hu2016network}.

In \cite{gale2019state}, the authors compare methods using more advanced heuristics to the approach of pruning based on weight magnitude and find that the more complex heuristics do not provide an advantage in practice. This is in line with a number of recent papers that have gone back to using weight magnitude and have focused instead on changing how and when pruning is applied to the network~\cite{gomez2019learning,frankle2019lottery,frankle2019stabilizing}, which has arguably proved to be a more important concern. Our proposed method follows this approach.

~

On the whole, the most closely related method to Gibbs pruning is targeted dropout~\cite{gomez2019learning}, which is also an in-training pruning method based on weight magnitude that can be used for structured or unstructured pruning. Like our proposed method, it takes inspiration from stochastic regularization methods like dropout in stochastically pruning during training to make the network robust under pruning. We compare our method to it experimentally in Section~\ref{chap:evaluation}.

\subsection{Retraining with Random Initialization}

In \cite{liu2019rethinking}, the authors made the surprising observation that when using many established pruning methods, if the pruned network weights were re-initialized randomly and trained from scratch, the resulting accuracy would match or surpass that of the pruned network. While existing pruning methods were effective at determining effective small architectures, it turned out to be more effective to train these architectures normally rather than starting with a large network and pruning it. This finding was a major challenge to the value of network pruning as a compression technique.

A number of recent papers have replicated this experiment with new pruning methods and have shown that they outperform retraining with random initialization, in particular methods like ours that focus more on co-adapting network weights and pruning masks during the training and pruning procedures~\cite{gomez2019learning,frankle2019stabilizing,gale2019state}. We also demonstrate this result with our proposed method in Section \ref{sec:ev-re-initialization}, making it part of a new generation of pruning methods that can be shown to outperform random re-initialization.

\section{Gibbs Pruning}
\label{chap:theory}

Gibbs pruning is a layer-wise pruning method, meaning that every layer in a neural network is pruned independently. Note that we do not prune regularization or pooling layers, and we may omit other layers from pruning if they have very few weights and pruning them will significantly harm performance. Since Gibbs pruning acts on individual layers, we describe it in terms of a single layer. Let \(\w \in \Real^N\) represent the layer's weights, flattened into a single vector in an arbitrary order. Let \(\x \in \{-1,1\}^N\) represent a \textbf{pruning mask} for the layer, where \(x_i=-1\) means that \(w_i\) is pruned and \(x_i=1\) means that \(w_i\) is not pruned. We always use \(N\) to denote the number of weights in a layer.

In Gibbs pruning, \(\x\) is sampled from a \textbf{Gibbs distribution} \(p(\x)\) at every training step and the layer weights are either retained or masked out during the forward pass based on the corresponding values in \(\x\). The distribution is sampled once again after training to determine the final pruning mask. Note that weight values are not permanently modified during training, allowing a connection to be masked out in some iterations and active in others. Per-iteration mask sampling is illustrated in Figure \ref{fig:prob-pruning}.

A Gibbs distribution over \(\x\) has the probability mass function:
\begin{equation}
	p(\x) = \frac{1}{Z(\beta)}e^{-\beta H(\x)},
	\label{eq:gibbs}
\end{equation}
where \(\beta \in (0,\infty)\) is an inverse temperature parameter, which controls the overall tendency of the distribution to take a low-energy state, and \(Z(\beta)\) is a partition function, which is set to ensure that the probability mass function sums to 1.

Many applications of Gibbs distributions gradually increase \(\beta\) over many sampling iterations to move the distribution from a disordered state to an ordered state. This procedure is called \textbf{annealing}, by analogy to annealing heat treatments in metallurgy. The value of \(\beta\) is inversely proportional to a simulated system temperature, so annealing can be understood as gradually reducing the temperature of the system over time. Using annealing requires defining a schedule for changing \(\beta\) over time, which we call the \textbf{annealing schedule}.

The main idea behind our proposed method is to gradually converge to a pruning mask during the course of an ordinary neural network training procedure. This allows the weights of the neural network and the pruning mask to dynamically adapt to one another. We hypothesize that this approach is more effective than either training a smaller network or pruning a network that is already trained, even if fine-tuning is used. In particular, we want to train the network weights in such a way that they are resilient to pruning, which is something that is not generally done in pruning methods that prune trained networks and then perform fine-tuning.

Gibbs distributions are a natural fit for this kind of pruning. They permit annealing, which allows for gradual convergence from unpredictable behaviour to a final steady state in a controllable way. They are also a very general family of probability distributions, giving considerable flexibility in defining pruning methods. For instance, every Markov random field corresponds to a Gibbs distribution, and vice versa~\cite{geman1984stochastic}.

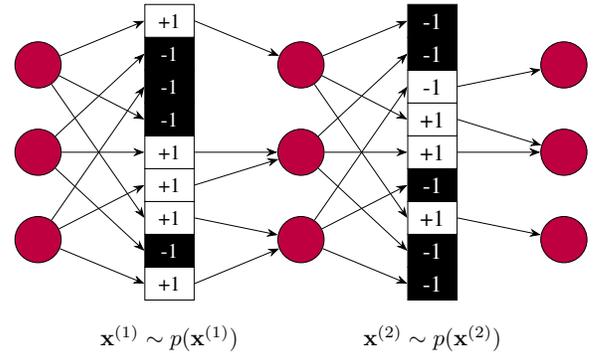
\begin{figure}[t!]
	\centering
	\begin{adjustbox}{width=0.42\textwidth}
	\begin{tikzpicture}[scale=1.0]
		\node[circle,minimum size=20,draw,fill=purple](1) at (0,0) {};
		\node[circle,minimum size=20,draw,fill=purple](2) at (0,1.33) {};
		\node[circle,minimum size=20,draw,fill=purple](3) at (0,2.66) {};
		\node[rectangle,minimum width=0.75cm,minimum height=0.5cm,draw,fill=white](s1)  at (2,-0.67) {\small +1};
		\node[rectangle,minimum width=0.75cm,minimum height=0.5cm,draw,fill=black,text=white](s2)  at (2,-0.17) {\small -1};
		\node[rectangle,minimum width=0.75cm,minimum height=0.5cm,draw,fill=white](s3)  at (2,0.33) {\small +1};
		\node[rectangle,minimum width=0.75cm,minimum height=0.5cm,draw,fill=white](s4)  at (2,0.83) {\small +1};
		\node[rectangle,minimum width=0.75cm,minimum height=0.5cm,draw,fill=white](s5)  at (2,1.33) {\small +1};
		\node[rectangle,minimum width=0.75cm,minimum height=0.5cm,draw,fill=black,text=white](s6)  at (2,1.83) {\small -1};
		\node[rectangle,minimum width=0.75cm,minimum height=0.5cm,draw,fill=black,text=white](s7)  at (2,2.33) {\small -1};
		\node[rectangle,minimum width=0.75cm,minimum height=0.5cm,draw,fill=black,text=white](s8)  at (2,2.83) {\small -1};
		\node[rectangle,minimum width=0.75cm,minimum height=0.5cm,draw,fill=white](s9)  at (2,3.33) {\small +1};
		\draw [-Stealth] (1) -- (s1.west); \draw [-Stealth] (1) -- (s4.west); \draw [-Stealth] (1) -- (s7.west);
		\draw [-Stealth] (2) -- (s2.west); \draw [-Stealth] (2) -- (s5.west); \draw [-Stealth] (2) -- (s8.west);
		\draw [-Stealth] (3) -- (s3.west); \draw [-Stealth] (3) -- (s6.west); \draw [-Stealth] (3) -- (s9.west);
		\node[circle,minimum size=20,draw,fill=purple](6) at (4,0) {};
		\node[circle,minimum size=20,draw,fill=purple](7) at (4,1.33) {};
		\node[circle,minimum size=20,draw,fill=purple](8) at (4,2.66) {};
		\draw [-Stealth] (s1.east) -- (6); \draw [-Stealth] (s3.east) -- (6);
		\draw [-Stealth] (s4.east) -- (7); \draw [-Stealth] (s5.east) -- (7);
		\draw [-Stealth] (s9.east) -- (8);
		\node[rectangle,minimum width=0.75cm,minimum height=0.5cm,draw,fill=black,text=white](s10)  at (6,-0.67) {-1};
		\node[rectangle,minimum width=0.75cm,minimum height=0.5cm,draw,fill=black,text=white](s11)  at (6,-0.17) {-1};
		\node[rectangle,minimum width=0.75cm,minimum height=0.5cm,draw,fill=white](s12)  at (6,0.33) {+1};
		\node[rectangle,minimum width=0.75cm,minimum height=0.5cm,draw,fill=black,text=white](s13)  at (6,0.83) {-1};
		\node[rectangle,minimum width=0.75cm,minimum height=0.5cm,draw,fill=white](s14)  at (6,1.33) {+1};
		\node[rectangle,minimum width=0.75cm,minimum height=0.5cm,draw,fill=white](s15)  at (6,1.83) {+1};
		\node[rectangle,minimum width=0.75cm,minimum height=0.5cm,draw,fill=white](s16)  at (6,2.33) {-1};
		\node[rectangle,minimum width=0.75cm,minimum height=0.5cm,draw,fill=black,text=white](s17)  at (6,2.83) {-1};
		\node[rectangle,minimum width=0.75cm,minimum height=0.5cm,draw,fill=black,text=white](s18)  at (6,3.33) {-1};
		\draw [-Stealth] (6) -- (s10.west); \draw [-Stealth] (6) -- (s13.west); \draw [-Stealth] (6) -- (s16.west);
		\draw [-Stealth] (7) -- (s11.west); \draw [-Stealth] (7) -- (s14.west); \draw [-Stealth] (7) -- (s17.west);
		\draw [-Stealth] (8) -- (s12.west); \draw [-Stealth] (8) -- (s15.west); \draw [-Stealth] (8) -- (s18.west);
		\node[circle,minimum size=20,draw,fill=purple](9) at  (8,0) {};
		\node[circle,minimum size=20,draw,fill=purple](10) at (8,1.33) {};
		\node[circle,minimum size=20,draw,fill=purple](11) at (8,2.66) {};
		\draw [-Stealth] (s12.east) -- (9);
		\draw [-Stealth] (s14.east) -- (10); \draw [-Stealth] (s15.east) -- (10);
		\draw [-Stealth] (s16.east) -- (11);

		\node[] at (2,-1.5) {\(\x^{(1)} \sim p(\x^{(1)})\)};
		\node[] at (6,-1.5) {\(\x^{(2)} \sim p(\x^{(2)})\)};
	\end{tikzpicture}
	\end{adjustbox}
	\caption{Connection pruning in a neural network using probability distributions over pruning masks.}
	\label{fig:prob-pruning}
\end{figure}

This high-level description provides a framework in which a wide range of pruning methods could be expressed.
In order to define an exact method, three elements need to be determined:
\begin{itemize}
	\item The Hamiltonian \(H(\x)\);
	\item An approach for annealing \(\beta\) during training;
	\item A method for sampling from the distribution.
\end{itemize}

\subsection{Properties of Gibbs Distributions}
\label{sec:gp-gibbs-properties}

As previously stated, Gibbs distributions can be annealed from behaving unpredictably to converging to a known final state. To be precise, at a high temperature or as \(\beta \rightarrow 0\), the distribution is effectively uniform over all \(\x\), and at a low temperature or as \(\beta \rightarrow \infty\), samples from the distribution converge to only those \(\x\) values that minimize \(H(\x)\)~\cite{geman1984stochastic}.

In order to define the Hamiltonian, we consider what the final converged state of the pruning mask should be. We can use some heuristic to determine a pruning mask that is likely be effective and set our Hamiltonian so that this state minimizes \(H(\x)\), which allows us to ensure that the annealing procedure will arrive at that state as \(\beta \rightarrow \infty\).

Aside from this constraint, the Hamiltonian could take any form, and could incorporate information like the current network parameters, the input data or intermediate values in the network, the initial network parameters, etc. While past pruning methods have considered a range of similar factors, we follow a number of contemporary works that find that simply looking at current weight magnitudes is enough to effectively prune networks~\cite{gale2019state,frankle2019lottery,han2016deep,guo2016dynamic}, and can outperform more complex metrics~\cite{gale2019state}. So, our Hamiltonians are only functions of the current weight magnitudes, the mask \(\x\), and some hyperparameters that are held constant. In Sections \ref{sec:gp-unstruct} and \ref{sec:gp-struct}, we explore possible specific functions to use for \(H(\x)\), and we test their efficacies experimentally in Section \ref{chap:evaluation}.

We perform annealing by increasing \(\beta\) from a low value to a high value while training. At high temperatures (low \(\beta\)), \(\x\) is roughly uniformly distributed over \(\{-1,1\}^N\), which acts similarly to stochastic regularization methods and in particular dropconnect~\cite{wan2013regularization}. This phase regularizes the network and conditions it to be robust under weight pruning. Once annealed to a low temperature (high \(\beta\)), the Gibbs distribution converges towards our desired final pruning mask. A significant amount of later training time can therefore be spent adapting to the particular structure of the pruning mask. The balance between time spent on adapting to random masks and time spent on adapting to the final mask can be controlled by the particular annealing schedule for \(\beta\) used in training.

In general, Gibbs distributions are difficult to sample from. This is because the partition function \(Z(\beta)\) is a sum containing a term for each possible value of \(\x\), and therefore can grow exponentially in complexity in \(N\). However, for some Hamiltonians, the distribution simplifies in ways that enable easy sampling. For others, Markov chain Monte Carlo (MCMC) methods that allow sampling without explicitly calculating the partition function must be used. We propose sampling methods for each Hamiltonian we use in the following sections.

Another useful property of Gibbs distributions is that a Hamiltonian \(H(\x)=f(\x)+\gamma\) with some constant \(\gamma\in\Real\) produces the same distribution as the Hamiltonian \(H(\x)=f(\x)\). This property holds because \(e^{\beta H(\x)}\) turns the constant additive factor to a constant multiplicative factor, which is then cancelled out by the partition function.

\subsection{Definitions}
\label{sec:gp-defns}

\(p\in[0,1]\) represents the pruning fraction, that is the fraction of weights that are removed in the final pruned layer. 
For theoretical analysis, we require that \(p\) represent a pruning fraction that is achievable given the number of pruned weights or structures in the layer. For instance, if a layer contained four filters, pruning 80\% of filters (\(p=0.8\)) would be impossible. In practice, the number of weights is often large enough that achieving a pruning fraction of \(p\) can be closely approximated without having to exactly meet it.

We define an operation \(Q(p,\w)\) as follows. Let \(\v\) represent a vector with the elements of \(\w\) sorted by magnitude such that \(v_i^2 \leq v_{i+1}^2\). Given a pruning fraction \(p\) between 0 and 1, select \(Q(p,\w) = v_i^2\), where \(i = p(N-1)+1\). If \(i\) is not an integer for the given value of \(p\), we instead linearly interpolate between the two nearest defined points \(Q(p^-,\w)\) and \(Q(p^+,\w)\):
\begin{gather}
	\text{Let } p^- = \left\lfloor\frac{i-1}{N-1}\right\rfloor,\quad p^+ = \left\lceil\frac{i-1}{N-1}\right\rceil, \nonumber \\
    Q(p,\w) = Q(p^-,\w) + \frac{p - p^-}{p^+ - p^-}\left(Q(p^+,\w) - Q(p^-,\w)\right).
\end{gather}
In effect, $Q(p,\w)$ is  the empirical $p$th quantile of the squared magnitudes of \(\w\).

We define a neighbourhood as a set of indices corresponding to a subset of weights in the neural network layer. Let \(\N_k\) represent the \(k\)th neighbourhood and let \(M\) represent the total number of neighbourhoods. The neighbourhoods in a layer form a partition of the set of weight indices, i.e., neighbourhoods are disjoint and the union of all neighbourhoods is the set of all weight indices. In practice, we define neighbourhoods so as to correspond to structures like kernels or filters in a neural network layer.
We also define a vector \(\bar{\w}\) containing the root mean squared weight value for each neighbourhood:
\begin{equation}
\bar{w}_k^2 = \frac{1}{|\N_k|}
\sum_{i\in \N_k}w_i^2.
\end{equation}

We use \(\ind[\cdot]\) to represent the indicator function, which takes the value 1 if its argument is true and 0 if it is false.

\subsection{Unstructured Pruning}
\label{sec:gp-unstruct}

Given the flexibility of Gibbs distributions, we have a massive range of Hamiltonians that could possibly work for pruning. In this section, we justify our choices of Hamiltonians for unstructured pruning by mathematically expressing requirements for pruning based on weight magnitudes, listing some heuristic design principles to help narrow down our search, and exploring some families of functions that can satisfy our requirements.

We use the heuristic that weights with lower magnitude are less important to the network's behaviour. So for unstructured pruning with a given \(p\), we design our Hamiltonian to converge to a state where the fraction \(p\) of weights with lowest magnitude are pruned. Let \(\x_\cvg\) represent this state:
\begin{equation}
	\x_\cvg = \argmin_\x \left[\x^T\w^2\right] \text{ s.t. } \frac{\sum_i \ind[x_i = -1]}{N} = p,
\end{equation}
where \(\w^2\) represents the element-wise square of \(\w\). If there are multiple states satisfying these conditions, we select one arbitrarily.
To converge to this state, the Hamiltonian must therefore fulfil the requirement:
\begin{equation}
	\argmin_\x H(\x) = \x_\cvg.
\label{eq:unstruct_req}
\end{equation}
\(\x_\cvg\) can alternatively be expressed in terms of \(Q(p,\w)\):
\begin{equation}
	x_{\cvg,i} = -1 + 2\ind[w_i^2 > Q(p,\w)].
\end{equation}

We have a wide design latitude outside of this requirement. Design heuristics that guide our choice of Hamiltonians include:
\begin{itemize}
	\item Mathematical simplicity.
	\item Ease and speed of sampling from the resulting Gibbs distribution.
	\item Smooth convergence from random behaviour to the converged state. We hypothesize that this results in a better adapted network than sudden changes.
\end{itemize}
We now examine types of Hamiltonians that might make sense given these goals.

\subsubsection{Binary Hamiltonian}

An obvious simple Hamiltonian that fulfils (\ref{eq:unstruct_req}) is one that has a constant low energy when the mask is in the converged state, and a constant high energy otherwise:
\begin{equation}
	H(\x) =  \alpha\ind[\x \neq \x_\cvg] + \gamma, \quad \alpha > 0.
	\label{eq:bin_general}
\end{equation}
As previously discussed, constant additive factors in the Hamiltonian do not change the distribution, so \(\gamma\) does not matter. We can also combine \(\alpha\) into \(\beta\) when selecting our annealing schedule. Therefore, without loss of generality, we can consider:
\begin{equation}
	H(\x) =  \ind[\x \neq \x_\cvg].
	\label{eq:bin}
\end{equation}

Sampling from this distribution is easy, since it can be done by selecting \(\x_\cvg\) with some probability \(p_{\text{cvg}}\), and sampling uniformly from \(\{-1,1\}^N\) with probability \(1-p_{\text{cvg}}\). The exact probability is as follows.
\begin{lemma}
	Using a sampler for (\ref{eq:bin}) that selects \(\x_\cvg\) with probability \(p_{\text{cvg}}\) and  otherwise samples uniformly from \(\{-1,1\}^N\), \(p_\cvg = \frac{1-e^{-\beta}}{(2^N-1)e^{-\beta}+1}\).
	\label{lemma:bin-p_cvg}
\end{lemma}

\begin{proof}
	The partition function normalizes the distribution, so it is:
	\begin{align}
		Z(\beta) &= \sum_{\x\in\{-1,1\}^N}e^{-\beta \ind[\x\neq\x_\cvg]} \nonumber \\
				 &= (2^N-1)e^{-\beta} + 1.
	\end{align}
	The probability of sampling \(\x_\cvg\) is therefore:
	\begin{equation}
		P[\x=\x_\cvg] = \frac{1}{(2^N-1)e^{-\beta}+1}.
	\end{equation}
	Solving for \(p_\cvg\) gives:
	\begin{align}
		p_\cvg + \frac{1}{2^N}(1-p_\cvg) &= P[\x=\x_\cvg] \\
		p_\cvg &= \frac{1}{1-\frac{1}{2^N}}\left(\frac{1}{(2^N-1)e^{-\beta}+1}  - \frac{1}{2^N}\right) \nonumber \\
			   &= \frac{1-e^{-\beta}}{(2^N-1)e^{-\beta}+1}
	\end{align}
\end{proof}

\subsubsection{Linear Hamiltonians}

Another simple family of Hamiltonians that can meet our requirements are those that are linear in \(\x\):
\begin{equation}
	H(\x) = \a^T\x, \a \in \mathbb{R}^N.
	\label{eq:lin_family}
\end{equation}
Because they can take values other than zero and one, linear Hamiltonians offer more flexibility than a binary Hamiltonian and might provide smoother convergence.

A vector \(\x\) minimizes the Hamiltonian if \(x_i=1\) when \(a_i<0\) and \(x_i=-1\) when \(a_i>0\). So, to fulfil (\ref{eq:unstruct_req}), \(\a\) must be chosen such that \(a_i < 0\) when \(x_{\cvg,i} = 1\) and \(a_i > 0\) when \(x_{\cvg,i} = -1\). In terms of \(Q(p,\w)\), the following must hold:
\begin{equation}
	\sgn(a_i) = \sgn(Q(p,\w) - w_i^2).
\end{equation}

Since only the sign of \(\a\) is constrained and not its magnitude, we propose three possible schemes for selecting \(\a\) based on our previously stated design considerations:
\begin{enumerate}[label=(\roman*)]
	\item \hfill%
		\(a_i = \sgn(Q(p,\w) - w_i^2),\)
			  \hfill\refstepcounter{equation}\textup{(\theequation)}%
		\label{eq:a1}
	\item \hfill%
		\(a_i = Q(p,\w) - w_i^2,\)
			  \hfill\refstepcounter{equation}\textup{(\theequation)}%
		\label{eq:a2}
	\item \hfill%
		\(a_i = \sqrt{Q(p,\w)} - |w_i|.\)
			  \hfill\refstepcounter{equation}\textup{(\theequation)}%
		\label{eq:a3}
\end{enumerate}

In the formulation given by (\ref{eq:a1}) each vector element contributes either -1 or 1. Equations (\ref{eq:a2}) and (\ref{eq:a3}) instead cause \(a_i\) values to have higher or lower magnitudes depending on how far the corresponding weight magnitudes are from the \(p\)th quantile. We hypothesize that this will be more effective, since weights with very large or small magnitudes will cause the corresponding pruning mask elements to converge quickly to their final state, whereas more borderline weights will cause the corresponding pruning mask elements to behave more randomly, effectively holding off longer on the decision of whether to prune those weights until their importance becomes more clear. Equations (\ref{eq:a2}) and (\ref{eq:a3}) only differ in whether squared weight values or absolute weight values are used.

\subsubsection{Sampling}

Although Gibbs distributions are difficult to sample from generally, using a linear Hamiltonian causes the elements of \(\x\) to be independent, meaning they can easily be sampled individually. In the following proposition, we show that our proposed Hamiltonian for unstructured pruning factors in terms of elements of \(x\).
\begin{proposition}
The Gibbs distribution corresponding to a linear Hamiltonian has a product form.
\label{prop:unstruct-factor}
\end{proposition}
\begin{proof}
	The partition function \(Z(\beta)\) normalizes the distribution. With a linear Hamiltonian (\ref{eq:lin_family}), it has the form:
\begin{align*}
    Z(\beta) &= \sum_{\x\in\{-1,1\}^N} e^{-\beta \sum_{i=1}^N a_ix_i} \\
    &= \sum_{x_1\in\{-1,1\}} \sum_{x_2\in\{-1,1\}} ... \sum_{x_N\in\{-1,1\}} \prod_{i=1}^N e^{-\beta a_ix_i} \\
    &= \prod_{i=1}^N \left(e^{-\beta a_i} + e^{\beta a_i} \right) \\
    &= \prod_{i=1}^N 2 \mbox{cosh} \Big(\beta a_i\Big).
\end{align*}
The Gibbs distribution therefore factors as:
\begin{align*}
p(\x) &= \frac{1}{\prod_{i=1}^N 2 \mbox{cosh} \Big(\beta a_i\Big)} e^{-\beta \sum_{i=1}^N a_ix_i} \\
    &= \prod_{i=1}^N \frac{1}{2 \mbox{cosh} \Big(\beta a_i\Big)} e^{-\beta a_ix_i}.
\end{align*}
\end{proof}

\subsection{Structured Pruning}
\label{sec:gp-struct}

For structured pruning, we specifically consider the problem of pruning structured groups, or neighbourhoods, of weights which must be either all pruned or all kept. For instance, to do filter-wise pruning on a convolutional network, we would define the neighbourhoods to correspond to the filters of the layer. Following our approach of pruning based on weight magnitude, we propose pruning the neighbourhoods with lowest total squared weight magnitude, rather than considering individual weights. This mean that \(\x_\cvg\) is instead defined as:
\begin{multline}
	\x_\cvg = \argmin_\x \left[\x^T\w^2\right] \text{ s.t. } \frac{\sum_i \ind[x_i = -1]}{N} = p \text{ and } \\ (i,j \in \N_k) \implies x_{\cvg,i} = x_{\cvg,j}.
\end{multline}
An equivalent definition in terms of \(Q(p,\bar{\w})\) is:
\begin{equation}
	i \in \N_k \implies x_{\cvg,i} = -1 + 2\ind[\bar{w}_k^2 > Q(p,\bar{\w})].
\end{equation}

A binary Hamiltonian can be used exactly as described above, just with this new \(\x_\cvg\). A linear Hamiltonian can also be used with \(a_i\) values that are 1 or -1 depending on the converged state of each \(x_i\):
\begin{equation}
	H(\x) = \sum_{k=1}^M\sgn(Q(p,\bar{\w})-\bar{w}_k^2)\sum_{i\in\N_k}x_i.
	\label{eq:struct_ham_lin}
\end{equation}
This is analogous to the unstructured Hamiltonian (\ref{eq:a1}). The same approaches as previously discussed can be used for sampling these Hamiltonians for structured pruning.

\subsubsection{Quadratic Hamiltonian}

In structured pruning, the converged state is a function of average or total weight magnitudes over neighbourhoods. This means that setting linear coefficients based on the final converged state leads to coefficients that do not consider individual weight magnitudes, as the Hamiltonians (\ref{eq:a2}) and (\ref{eq:a3}) for unstructured pruning do. To both express structure requirements and consider the magnitudes of individual weights, we move to a quadratic family:
\begin{equation}
	H(\x) = \x^T\A\x + \b^T\x, \, \A \in \mathbb{R}^{N\times N}, \, \b \in \mathbb{R}^N.
\end{equation}
This Hamiltonian is equivalent to the Ising model in statistical physics~\cite{cipra1987introduction}. Without loss of generality, we assume \(\A\) is symmetrical. Since \(x_i \in \{-1,1\}\), the diagonal elements of \(\A\) only contribute constant terms, so we ignore them. The off-diagonal elements of \(\A\) add quadratic terms of the form \(a_{ij}x_ix_j\), which contribute an energy of \(a_{ij}\) if \(x_i\) and \(x_j\) take the same value and \(-a_{ij}\) if they are different. Therefore, the off-diagonal elements encourage two elements of \(\x\) to either take the same value in sampling or to take different values. This allows the quadratic terms to express desired structures. Since they do not distinguish between the \(+1\) and \(-1\) states, only depending on whether two elements have the same state, the linear terms must be used for promoting or discouraging the pruning of particular elements.

Our desired structure requires mask elements in the same neighbourhood to take the same value. There are no cases where we require particular mask elements to take different values. Therefore, we set \(\A\) such that \(a_{ij}<0\) if \(i\) and \(j\) are in the same neighbourhood, and \(a_{ij}=0\) otherwise.
For simplicity, we use a constant \(a_{ij}=-c/2\) for non-zero elements of \(A\). This leads to the Hamiltonian:
\begin{equation}
	H(\x) = -c\sum_{k=1}^M\,\sum_{\substack{i,j \in \N_k\\i \neq j}}x_ix_j  + \b^T\x.
	\label{eq:struct_ham_b}
\end{equation}

To prove that the converged state obeys the requirements of structured pruning, we introduce the following property:

\begin{proposition}
	Given (\ref{eq:struct_ham_b}), for any \(\b \in \mathbb{R}^N\), we can select a value \(c\) such that, in the state minimizing \(H(\x)\), all elements in the same neighbourhood take the same value.
	\label{prop:big-enough-c}
\end{proposition}

\begin{proof}
	Let \(N_\text{pairs}\) represent the number of pairs of weights in the same neighbourhood:
	\begin{equation}
		N_\text{pairs} = |\{\{i,j\} \mid \exists k, i,j \in \N_k\}|.
	\end{equation}
	Select a positive value \(c > (\max_\x \b^T\x) - (\min_\x \b^T\x)\). A state where all elements in each neighbourhood take the same value therefore has maximum energy:
	\begin{equation}
		H(\x) \leq -cN_\text{pairs} + \max_\x \b^T\x.
		\label{eq:min_c_proof}
	\end{equation}
	If there exist two elements in a neighbourhood with different values, then at least one \(-cx_ix_j\) term must switch from adding \(-c\) to the energy to adding \(+c\). Therefore, a state where some elements in a neighbourhood take different values has minimum energy:
	\begin{align}
		H(\x) &\geq -c(N_\text{pairs}-2) + \min_\x \b^T\x \nonumber \\
			  &\geq -cN_\text{pairs} +2((\max_\x \b^T\x) - (\min_\x \b^T\x)) + \min_\x \b^T\x \nonumber \\
			  &\geq -cN_\text{pairs} + (\max_\x \b^T\x) + c.
	\end{align}
	This is strictly greater than (\ref{eq:min_c_proof}) and therefore cannot minimize \(H(\x)\).
\end{proof}

This proposition shows that when selecting \(\b\), we can tune \(c\) in a way that lets us limit our consideration to states where neighbourhoods take the same values. Within a neighbourhood \(\N_k\), the linear contribution to the Hamiltonian is therefore \(\pm\bar{b}_k\) where \(\bar{b}_k = \sum_{i\in\N_k} b_i\). If this value is positive, then the neighbourhood will be pruned in the minimum state and if it is negative, the neighbourhood will not be pruned. Our definition of \(\x_\cvg\) therefore requires:
\begin{align}
\begin{split}
	\bar{b}_k < 0 &\text{ if } \bar{w}_k^2 > Q(p,\bar{\w}), \\
	\bar{b}_k > 0 &\text{ if } \bar{w}_k^2 < Q(p,\bar{\w}).
	\label{eq:bar_b_reqs}
\end{split}
\end{align}

To satisfy this requirement and incorporate individual weight magnitudes, we set \(\b\) in a way analogous to (\ref{eq:a2}):
\begin{equation}
	b_i =  Q(p,\bar{\w}) - w_i^2.
\end{equation}
The corresponding Hamiltonian is then:
\begin{equation}
	H(\x) = -c\sum_{k=1}^M\,\sum_{\substack{i,j \in \N_k\\i \neq j}}x_ix_j  + \sum_{i=1}^N\left(Q(p,\bar{\w}) - w_i^2\right)x_i.
	\label{eq:struct_ham}
\end{equation}
This formulation may perform better than the linear structured pruning Hamiltonian (\ref{eq:struct_ham_lin}) since it takes the magnitudes of individual weights into account before converging to a structured mask. We test both and compare results in the following section.

The hyperparameter \(c\) must be tuned for the particular network and task. Setting it too low will prevent the mask from converging to a state where neighbourhoods take the same value, meaning that structured pruning will not be achieved. On the other hand, setting \(c\) very high will result in the neighbourhood constraint always being followed, even very early in training, which will make the convergence to the pruning mask less gradual. We find that this reduces the final accuracy in practice.

\subsubsection{Sampling}
\label{sec:struct-sampling}

Sampling from the Gibbs distribution with a quadratic Hamiltonian is more challenging than for binary or linear Hamiltonians. First of all, we note that the distribution factors in terms of neighbourhoods and so individual neighbourhoods can be sampled independently:

\begin{proposition}
The Gibbs distribution corresponding to the Hamiltonian (\ref{eq:struct_ham}) has a block product form.
\end{proposition}
\begin{proof}
Since the first term of the Hamiltonian (\ref{eq:struct_ham}) is zero for elements from different neighbourhoods, the Hamiltonian can be expressed as:
\begin{align}
    H(\x) &= \sum_{k=1}^M H_k(\x_k),
\end{align}
where \(H_k(\x_k)\) is the Hamiltonian (\ref{eq:struct_ham}) only computed over the elements of the neighbourhood \(\N_k\). Following the same approach as for Proposition~\ref{prop:unstruct-factor}, the partition function can then be shown to factor as:
\begin{align}
    Z(\beta) &= \prod_{k=1}^M\,\sum_{\x_k \in \{-1,1\}^{|\N_k|}} e^{-\beta H_k(\x_k)} \nonumber \\
    &= \prod_{k=1}^M Z_k(\beta),
\end{align}
and the overall distribution factors as:
\begin{align}
p(\x) &= \prod_{k=1}^M\frac{1}{Z_k(\beta)} e^{-\beta H_k(\x_k)}.
\end{align}
\end{proof}

If neighbourhoods are small enough, the partition function for each can be computed and the distribution can be sampled from directly. But for large neighbourhoods, this is computationally infeasible, and so we use a MCMC method to generate samples. A sampling method that parallelizes well and is simple to implement on a GPU is preferable for neural network training to highly iterative methods like standard Gibbs sampling or the Wolff algorithm~\cite{wolff1989collective}. We therefore propose using Chromatic Gibbs sampling~\cite{gonzalez2011parallel}.

Chromatic Gibbs sampling involves making a colouring over the elements in \(\x\) such that any interacting elements \(x_i\) and \(x_j\) in the Hamiltonian have different colours. For our Hamiltonians, \(x_i\) and \(x_j\) interact if there is a term \(a_{ij} x_ix_j\) with $a_{ij} \neq 0$. An iteration of the Markov chain then consists of sampling all elements of one colour simultaneously given the current values of all other elements.

To create a colouring, we arbitrarily divide each neighbourhood into two sets of elements and modify the Hamiltonian to remove all quadratic terms containing elements in the same set. In graphical terms, this transforms the Markov random field for each neighbourhood from a complete graph to a complete bipartite graph that can be coloured with two colours, as shown in Figure~\ref{fig:chromatic}. This modification still preserves a high degree of connectivity within each neighbourhood and so is still effective at encouraging elements in neighbourhoods to take the same value. For instance, to perform filter-wise pruning of a convolutional network, we can divide all connections into two groups based on their input channel. 

\begin{figure}[t!]
	\centering
	\begin{adjustbox}{width=0.35\textwidth}
        	\begin{tikzpicture}[scale=1.5]
		\node[circle,minimum size=15,draw,fill=yellow](1) at (3.5,0) {};
		\node[circle,minimum size=15,draw,fill=yellow](2) at (3.5,0.6) {};
		\node[circle,minimum size=15,draw,fill=yellow](3) at (3.5,1.2) {};
		\node[circle,minimum size=15,draw,fill=yellow](4) at (3.5,1.8) {};
		\node[circle,minimum size=15,draw,fill=blue](6) at (5,0) {};
		\node[circle,minimum size=15,draw,fill=blue](7) at (5,0.6) {};
		\node[circle,minimum size=15,draw,fill=blue](8) at (5,1.2) {};
		\node[circle,minimum size=15,draw,fill=blue](9) at (5,1.8) {};

		\node[circle,minimum size=15,draw](11) at (0.63,0) {};
		\node[circle,minimum size=15,draw](12) at (0,0.6) {};
		\node[circle,minimum size=15,draw](13) at (0,1.2) {};
		\node[circle,minimum size=15,draw](14) at (0.63,1.8) {};
		\node[circle,minimum size=15,draw](16) at (1.38,0) {};
		\node[circle,minimum size=15,draw](17) at (2,0.6) {};
		\node[circle,minimum size=15,draw](18) at (2,1.2) {};
		\node[circle,minimum size=15,draw](19) at (1.38,1.8) {};
		\draw (1) -- (6); \draw (1) -- (7); \draw (1) -- (8); \draw (1) -- (9);
		\draw (2) -- (6); \draw (2) -- (7); \draw (2) -- (8); \draw (2) -- (9);
		\draw (3) -- (6); \draw (3) -- (7); \draw (3) -- (8); \draw (3) -- (9);
		\draw (4) -- (6); \draw (4) -- (7); \draw (4) -- (8); \draw (4) -- (9);

		\draw (11) -- (16); \draw (11) -- (17); \draw (11) -- (18); \draw (11) -- (19);
		\draw (12) -- (16); \draw (12) -- (17); \draw (12) -- (18); \draw (12) -- (19);
		\draw (13) -- (16); \draw (13) -- (17); \draw (13) -- (18); \draw (13) -- (19);
		\draw (14) -- (16); \draw (14) -- (17); \draw (14) -- (18); \draw (14) -- (19);
		\draw (11) -- (12); \draw (11) -- (13); \draw (11) -- (14);
		\draw (12) -- (11); \draw (12) -- (13); \draw (12) -- (14);
		\draw (13) -- (11); \draw (13) -- (12); \draw (13) -- (14);
		\draw (14) -- (11); \draw (14) -- (12); \draw (14) -- (13);
		\draw (16) -- (17); \draw (16) -- (18); \draw (16) -- (19);
		\draw (17) -- (16); \draw (17) -- (18); \draw (17) -- (19);
		\draw (18) -- (16); \draw (18) -- (17); \draw (18) -- (19);
		\draw (19) -- (16); \draw (19) -- (17); \draw (19) -- (18);

		\draw[-Stealth,thick] (2.25,0.9) -- (3.25, 0.9);
	\end{tikzpicture}
	\end{adjustbox}
	\captionsetup{font=small}
	\caption[Removing interactions to allow a 2-colouring of a random field.]{Removing interactions between elements so that neighbourhoods form a complete bipartite graph allows a 2-colouring, which in turn allows chromatic Gibbs sampling to be used.}
	\label{fig:chromatic}
\end{figure}
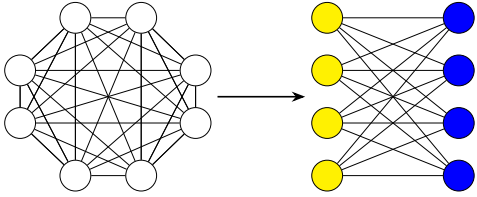

When using MCMC sampling, it is important to initialize the chain at a plausible configuration, since later samples may still be somewhat correlated with the initial state~\cite{brooks2011handbook}. To initialize the chain, we sample from an approximation of the quadratic Hamiltonian that ignores states that violate the structure requirement. This makes the quadratic term a constant value, leading to the linear Hamiltonian:
\begin{equation}
	H(\bar{\x}) = C + \sum_{k=1}^M |\N_k|Q(p,\bar{\w}) - \bar{\w}_k^2)\bar{x}_k,
	\label{eq:struct_ham_approx}
\end{equation}
where \(\bar{x}_k\) is the value taken by all \(x_i\) in neighbourhood \(k\) and \(C\) is a constant representing the contribution of the quadratic terms. Since it is linear, this Hamiltonian is easy to sample from, and is likely to represent local modes of the distribution since the quadratic term penalizes states that do not obey the structure requirement.

\subsection{Schedule Stretching}
\label{sec:gp-schedule}

Many existing pruning methods make use of much longer training schedules than are typically used for training networks because they do additional fine-tuning (e.g., \cite{guo2016dynamic,han2016deep,li2017pruning}). Our pruning method aims to prune and train at the same time so we avoid explicit fine-tuning, but we still investigate whether a longer training schedule can provide benefits. To test this, we propose \textbf{schedule stretching}. During training, we follow schedules for setting the optimizer learning rate and the Gibbs distribution parameter \(\beta\) as a function of the current training epoch. Let \(n\) represent the current epoch number, starting from zero, and let \(\lambda(n)\) and \(\beta(n)\) represent the learning rate and \(\beta\) schedules, respectively. When using schedule stretching, we set a hyperparameter \(s\) to a positive integer value, and instead use \(\lambda(\lfloor \frac{n}{s} \rfloor)\) and \(\beta(\lfloor \frac{n}{s}\rfloor)\) to set the learning rate and \(\beta\) at each epoch. The total number of training epochs is increased by a factor of \(s\).

\section{Evaluation}
\label{chap:evaluation}

We compare Gibbs pruning to several other structured and unstructured pruning methods. The other methods were chosen to represent current practices and research, either being
common, well-established pruning method or ones that have recently shown exceptional results. The lack of standard experimental setups in network pruning make it difficult to judge which methods are state-of-the-art~\cite{blalock2020state}, especially in the case of structured pruning, but we do compare to best known results where possible.

We evaluate performance on ResNet neural networks~\cite{he2016deep}, which is a common approach in the network pruning literature~\cite{blalock2020state}. Nearly all of the parameters of these networks are in convolutional layers, which already are much more sparse than dense layers of similar dimensions and additionally employ weight sharing, making them more difficult to prune than dense layers. Pruning methods that are effective on networks like AlexNet~\cite{krizhevsky2012imagenet} with large dense layers often do not show the same performance on networks like ResNet, making them a more challenging test.

In particular, we train ResNet-20 and ResNet-56 with linear projection~\cite{he2016deep}. ResNet-56 was chosen because it is the most common ResNet variant used in evaluating pruning methods~\cite{blalock2020state}, making it easier to compare our method to others. ResNet-20 on the other hand is the smallest standard ResNet variant. This makes it a useful test of pruning methods because it is less likely to be overparameterized for particular datasets, and it can be used with high pruning rates to investigate the efficacy of extremely small neural networks.

We run our experiments with the CIFAR-10 dataset~\cite{krizhevsky2012cifar}. This is a standard image classification dataset that is widely used in evaluating pruning methods~\cite{blalock2020state}.

The networks are trained for 200 epochs using the Adam optimizer~\cite{kingma2014adam}, with a learning rate initially set to \(10^{-3}\) and reduced by a factor of 10 at epoch 80 and every 40 epochs thereafter. We use data augmentation during training as in~\cite{he2016deep}, randomly shifting images horizontally and vertically by up to 10\% and flipping images horizontally with 50\% probability. This achieves baseline top-1 accuracies of 90.7\% for ResNet-20 and 92.2\% for ResNet-56.

For all methods, we prune all convolutional layers except for the first one, following the recommendation in \cite{gale2019state}. This is because the input to the first convolutional layer only has three channels, leading to the first convolutional layer having very few parameters and having its performance easily harmed by high pruning rates. A precise summary of pruned and unpruned layers is shown in Table \ref{t:params}.

\begin{table}[t!]
	\centering
	\caption{The fractions of parameters in pruned and unpruned layers in the evaluated neural networks.}
	\begin{tabular}{cc|c|c}
		\multicolumn{2}{c}{} & ResNet-20 & ResNet-56 \\\hline
		\multirowcell{3}{\textbf{Unpruned} \\ \textbf{layers}} & First convolution & 0.16\% & 0.05\% \\
										 & Batch normalization & 1.01\% & 0.95\% \\
										 & Dense & 0.24\% & 0.08\%  \\\hline
		\makecell{\textbf{Pruned} \\ \textbf{layers}} & Other convolutions & 98.60\% & 98.93\% \\
	\end{tabular}
	\label{t:params}
\end{table}

For pruning methods that use additional epochs for fine-tuning, we use a training rate of \(10^{-5}\) after the initial training phase. We also test changing the lengths of fine-tuning schedules, to evaluate tradeoffs between the number of additional training epochs and the final accuracy.

\subsection{Unstructured Pruning Results}
\label{sec:ev-unstruct}

The Hamiltonians previously discussed for unstructured pruning are the binary Hamiltonian (\ref{eq:bin}) and the linear Hamiltonians given by (\ref{eq:a1}), (\ref{eq:a2}), and (\ref{eq:a3}). Initial and final \(\beta\) values, as well as the number of epochs to anneal over and the use of a logarithmic or linear schedule were manually tuned to maximize final accuracy. The best results achieved for each are shown in Table \ref{t:unstruct-hams}. Since the best result was observed with (\ref{eq:a2}), we use this formulation in the rest of our experiments.

\begin{table}[t!]
	\centering
	\caption[Comparison of Hamiltonians for unstructured pruning.]{Comparison of Hamiltonians for unstructured pruning. 90\% of weights in each pruned layer are pruned.}
	\label{t:results_unstruct}
	\begin{tabular}{cc|cc}
		Variant & \makecell{ResNet-20\\Accuracy} & \makecell{ResNet-56\\Accuracy} \\\hline
		Binary (\ref{eq:bin}) & 86.6 & 89.1 \\
		Linear (\ref{eq:a1}) & 85.7 & 87.7 \\
		Linear (\ref{eq:a2}) & \textbf{87.5} & \textbf{89.8} \\
		Linear (\ref{eq:a3}) & 86.9 & 89.2
	\end{tabular}
	\label{t:unstruct-hams}
\end{table}

Notably, (\ref{eq:a2}) and (\ref{eq:a3}), which produced the best results, both create a Hamiltonian with coefficients that don't just express whether a weight should be pruned or not, but also vary depending on how far a weight is from \(Q(p,\w)\). This allows weights with very high or very low magnitudes to converge to a pruned or unpruned state earlier. This being desirable would also explain why (\ref{eq:a2}) is more effective than (\ref{eq:a3}), since it considers the squares of magnitude, making weights with high magnitudes vary from \(Q(p,\w)\) even more and therefore converge earlier.

After manually tuning \(\beta\) annealing for the Hamiltonian (\ref{eq:a2}), we arrived at a logarithmic schedule from 0.7 to 10000 done over the first 128 epochs of a 200 epoch schedule, updating \(\beta\) at the end of each epoch. This schedule is used for all subsequent unstructured pruning experiments.
In practice, we find that when using schedule stretching with very long schedules, the pruning mask does not always converge at \(\beta=10^4\). This indicates that the empirical distribution of weights can be significantly different in some highly stretched configurations. In these cases we increase the final beta to \(10^6\), which achieves convergence in all such cases.

\subsubsection{Comparison to Existing Methods}

We compare Gibbs pruning to four established unstructured pruning methods. For methods that use fine-tuning, we evaluate spending different amounts of time on fine-tuning. We also test different overall training times for Gibbs pruning by using schedule stretching as described in Section \ref{sec:gp-schedule}.

The first established method is using \(l_1\) regularization. We use a penalty of 0.001, which was manually tuned to maximize final accuracy. After training, we mask the fraction \(p\) of weights in each layer with the lowest magnitude. \(l_1\) regularization is known to cause weights to become sparse~\cite{goodfellow2016deep}, and has been used as a simple point of comparison for unstructured pruning methods~\cite{gomez2019learning}.

The other pruning methods we compare to are more advanced approaches from the literature. The first of these is the method proposed in \cite{han2015learning} by Han et al. Once training is complete, a certain percentage of the weights with lowest magnitude are pruned and the network is fine-tuned for some time. This procedure is then repeated several times. We test different schedule lengths in order to compare this methods to other methods that can use different training schedule lengths. Before each fine-tuning period, we increase the pruning rate by 10\% up to 90\%, fine-tuning for \(n\) epochs each time, leading to a total of \(200+9n\) epochs.

We also compare to iterative magnitude pruning (IMP), proposed by Frankle et al.~\cite{frankle2019lottery} with rewinding~\cite{frankle2019stabilizing}. This method trains the network several times, pruning gradually more each time and then rewinding the network weights to the values they had after the first 500 training steps. To test different training times, we vary the number of times that the network is trained with intermediate pruning rates between the initial rate of 0\% and the final rate of 90\%.

Finally, we test targeted dropout, a method recently proposed by Gomez et al.\cite{gomez2019learning}. In this method, the fraction \(\gamma\) of weights with lowest magnitude are each omitted from the network with probability \(\alpha\) at each training step. After training, the fraction of weights \(p\) with lowest magnitude are permanently pruned.
We use the authors' most successful hyperparameter settings: \(\alpha=0.75,\gamma=0.9\). We also evaluate ramping targeted dropout, a variant where \(\alpha\) is gradually changed from 0 to \(p\) during training~\cite{gomez2019learning}.

\begin{figure}[t!]
	\centering
	\begin{subfigure}{0.48\textwidth}
	\includegraphics[width=\textwidth]{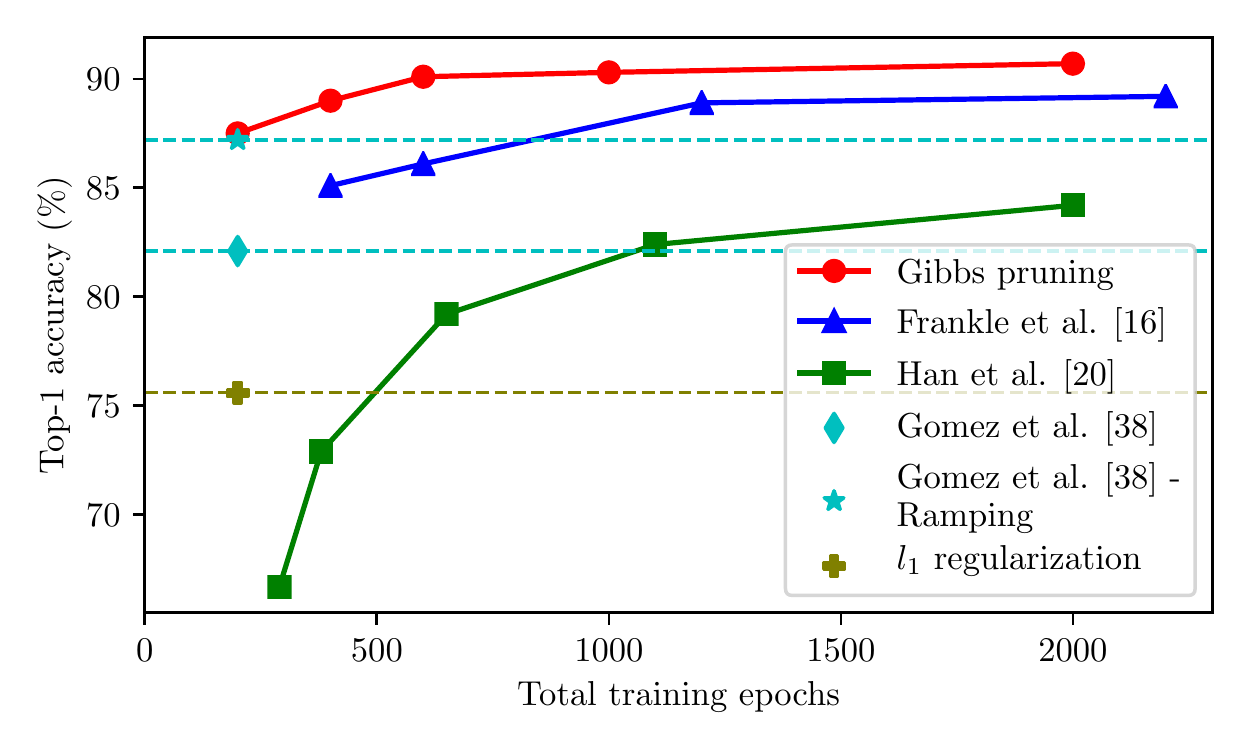}
	\label{fig:resnet20scatter}
	\vspace{-7mm}
	\caption{ResNet-20}
	\end{subfigure}
	\begin{subfigure}{0.48\textwidth}
	\includegraphics[width=\textwidth]{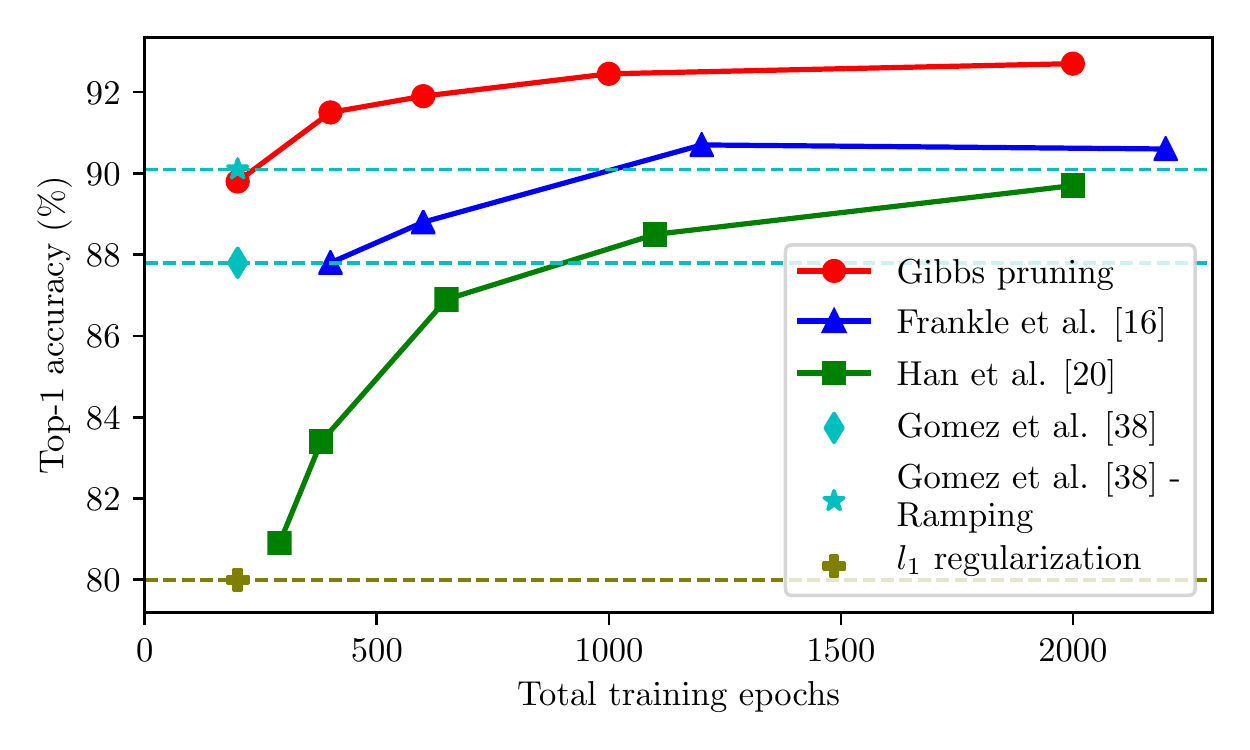}
	\label{fig:resnet56scatter}
	\vspace{-7mm}
	\caption{ResNet-56}
	\end{subfigure}
	\caption[Comparison of Gibbs pruning to other unstructured pruning methods.]{Comparison of Gibbs pruning to other unstructured pruning methods. Dotted lines are shown for comparison on methods that do not require additional training epochs. }
	\label{fig:resnetscatter}
\end{figure}

\begin{figure}[t!]
	\centering
	\includegraphics[width=0.48\textwidth]{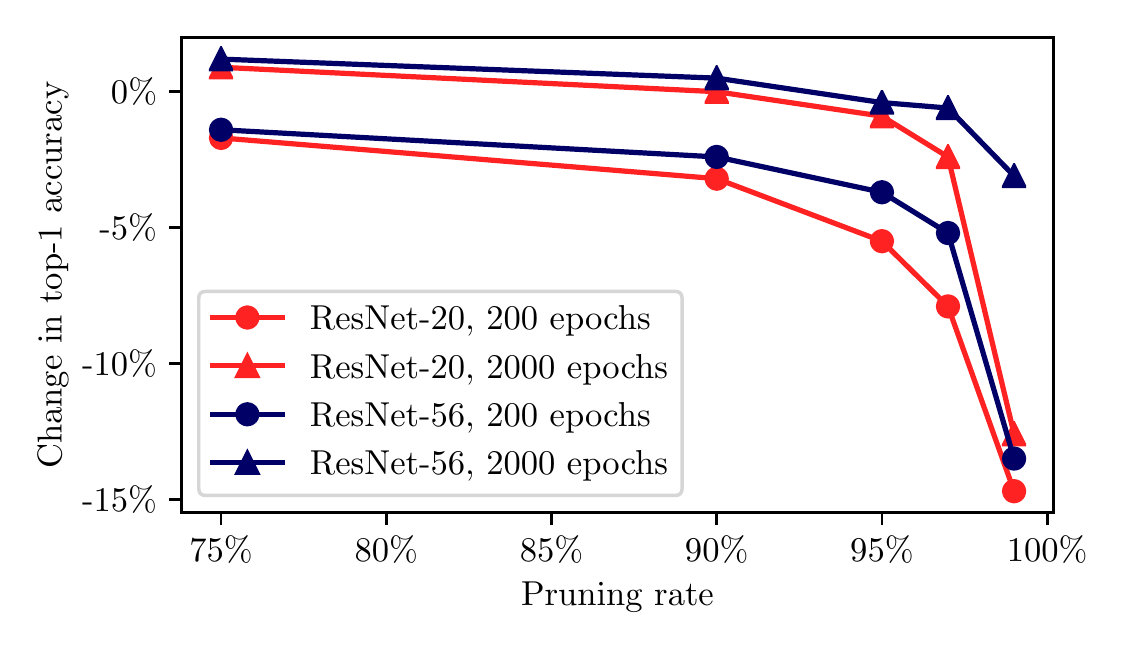}
	\caption{Loss in accuracy compared to the baseline versus pruning rate using unstructured Gibbs pruning.}
	\label{fig:p-acc}
\end{figure}

We evaluate pruning 90\% of weights in each pruned layer, which is a high degree of compression for the given networks. Because this pruning rate is challenging, it effectively shows performance differences between different methods. Results are shown in Figure~\ref{fig:resnetscatter}. These show that Gibbs pruning is effective without additional training epochs, giving similar results to the best results from the other evaluated methods. When the training time is extended using schedule stretching, Gibbs pruning decisively outperforms the other evaluated methods.

\subsubsection{High Pruning Rates}

We evaluate Gibbs pruning using different pruning rates in Figure~\ref{fig:p-acc}. When schedule stretching is used, the resulting accuracy can come close to or exceed the baseline accuracy even with very high pruning rates. For instance, removing 95\% of weights in pruned layers for an overall pruning rate of 94\% decreases accuracy by only 0.90\% for ResNet-20 and 0.45\% for ResNet-50. This corresponds to models with roughly 17\,000 final parameters when pruning ResNet-20 and 48\,000 final parameters when pruning ResNet-56. To the best of our knowledge, these are the best results reported for such a high compression rate on CIFAR-10~\cite{blalock2020state}.

\subsubsection{Random Masks and Re-initialization}
\label{sec:ev-re-initialization}

\begin{table}[t!]
	\centering
	\caption[Comparison of random masks and random weight re-initialization to Gibbs pruning.]{Comparison of random masks and random weight re-initialization to Gibbs pruning. Values are top-1 accuracy in percent.}
	\begin{subtable}{0.48\textwidth}
		\centering
		\caption{ResNet-20}
		\begin{tabular}{c|cccc}
			Method & \multicolumn{4}{c}{Pruning rate} \\\hline
				   & 50\% & 75\% & 90\% & 95\% \\\hline
				   Random mask & 89.3 & 87.2 & 83.2 & 79.8 \\
				   Random re-init. & 88.6 & 86.3 & 83.4 & 79.4 \\
				   Gibbs pruning & \textbf{89.9} & \textbf{89.0} & \textbf{87.5} & \textbf{85.2} \\
			   \end{tabular}
			   \vspace{2mm}
		   \end{subtable} \\
		   \begin{subtable}{0.48\textwidth}
			   \centering
			   \caption{ResNet-56}
			   \begin{tabular}{c|cccc}
				   Method & \multicolumn{4}{c}{Pruning rate} \\\hline
						  & 50\% & 75\% & 90\% & 95\% \\\hline
						   Random mask & \textbf{91.2} & 90.0 & 87.3 & 84.6 \\
						  Random re-init. & 90.3 & 89.1 & 86.7 & 84.5 \\
						  Gibbs pruning & \textbf{91.2} & \textbf{90.8} & \textbf{89.8} & \textbf{88.5} \\
					  \end{tabular}
				  \end{subtable}
				  \label{t:random}
	   \end{table}

As previously discussed, many pruning methods do not produce a higher accuracy than if the pruned network is re-initialized and retrained~\cite{liu2019rethinking}, calling into question the utility of training large networks and then pruning. To assess this phenomenon with Gibbs pruning, we evaluate retraining pruned networks with randomly re-initialized weights. For comparison, we also train networks from scratch using random pruning masks with the chosen sparsity.

Results are shown in Table \ref{t:random}. We find that random re-initialization does not meet the performance of our proposed method, and is more similar to training with a random mask. These results suggest that Gibbs pruning acts less as a search over possible network architectures, and instead adapts more dynamically along with weight values throughout training, resulting in a final representation that is more tuned to particular weights than other methods are.

\subsection{Structured Pruning Results}
\label{sec:ev-struct}

We also evaluate performance for structured convolutional neural network pruning. We consider pruning individual kernels as well as entire filters of size \(K\times K\times C\), where \(K\) is the kernel width and \(C\) is the number of input channels for the layer being pruned. Generally speaking, kernel-wise pruning can achieve higher accuracies for the same number of parameters, but might be harder to optimize performance for, compared to filter-wise pruning.

We compare the quadratic Hamiltonian (\ref{eq:struct_ham}) to the other Hamiltonians we proposed for structured pruning: the binary Hamiltonian (\ref{eq:bin}) and the linear Hamiltonian (\ref{eq:struct_ham_lin}). Results are shown in Table \ref{t:results_structured}. For both cases, the quadratic Hamiltonian produced the best result, so we use it for all other experiments.

In practice, we find that moving to kernel- or filter-wise pruning with the chosen Hamiltonians does not significantly affect the convergence of pruning masks as a function of \(\beta\). We therefore use the same \(\beta\) schedule as in unstructured pruning, annealing from 0.7 to 10\,000 over the first 128 epochs. Given this schedule, we manually tune \(c\). We find that a value around \(0.01\) produces the best final results, so this value is used for all subsequent experiments.

\begin{table}[t!]
	\centering
	\caption{Comparison of Hamiltonians for structured pruning.}
	\label{t:results_structured}
	\vspace{-2mm}
	\begin{subtable}{0.48\textwidth}
		\centering
		\caption{Kernel-wise pruning}
		\begin{tabular}{c|cc}
			Hamiltonian & \makecell{ResNet-20\\Accuracy} & \makecell{ResNet-56\\Accuracy} \\\hline
			Binary (\ref{eq:bin}) & 83.5 & 87.7 \\
			Linear (\ref{eq:struct_ham_lin}) & 82.0 & 86.5 \\
			Quadratic (\ref{eq:struct_ham}) & \textbf{84.8} & \textbf{88.2}
		\end{tabular}
		\vspace{2mm}
	\end{subtable} \\
	\begin{subtable}{0.48\textwidth}
		\centering
		\caption{Filter-wise pruning}
		\begin{tabular}{c|cc}
			Hamiltonian & \makecell{ResNet-20\\Accuracy} & \makecell{ResNet-56\\Accuracy} \\\hline
			Binary (\ref{eq:bin}) & 82.3 & 87.0 \\
			Linear (\ref{eq:struct_ham_lin}) & 79.7 & 86.3 \\
			Quadratic (\ref{eq:struct_ham}) & \textbf{83.5} & \textbf{87.5}
		\end{tabular}
	\end{subtable}
\end{table}

For kernel-wise Gibbs pruning, each neighbourhood is small enough that the partition function is practically computable, so we sample directly from the Gibbs distribution. For filter-wise pruning, we use chromatic Gibbs sampling as described in Section~\ref{sec:struct-sampling}. We run the Markov chain for 50 iterations, as further iterations were not observed to affect performance in practice.

\begin{figure}[t!]
	\centering
	\begin{subfigure}{0.48\textwidth}
	\includegraphics[width=\textwidth]{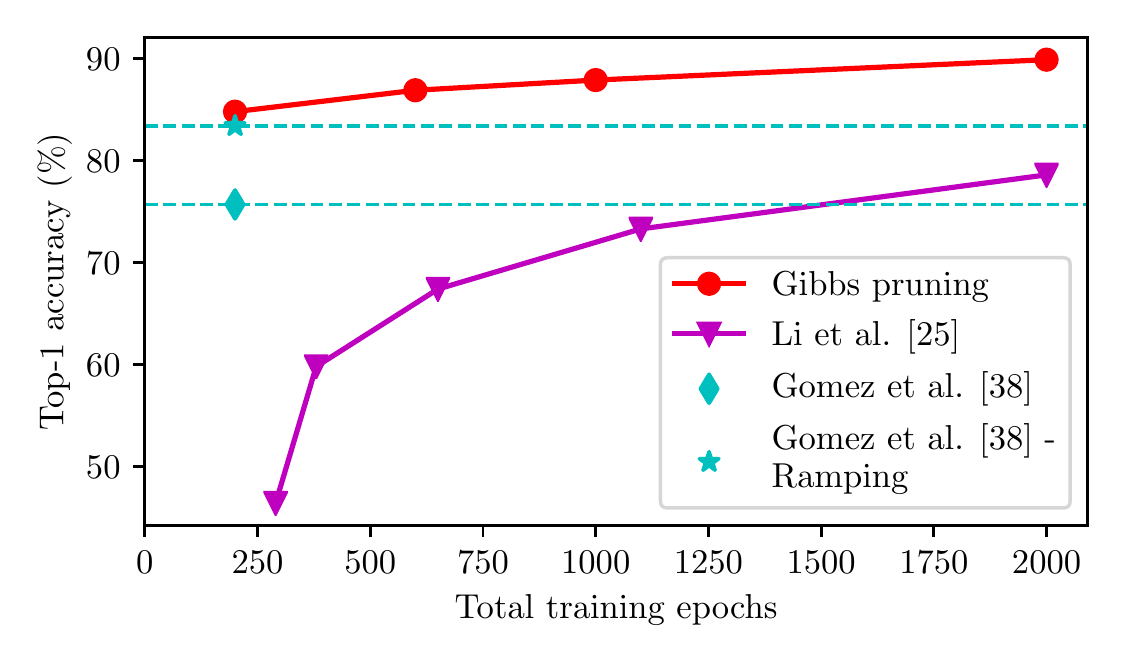}
	\label{fig:resnet20kernelscatter}
	\vspace{-7mm}
	\caption{ResNet-20}
	\end{subfigure}

	\begin{subfigure}{0.48\textwidth}
	\includegraphics[width=\textwidth]{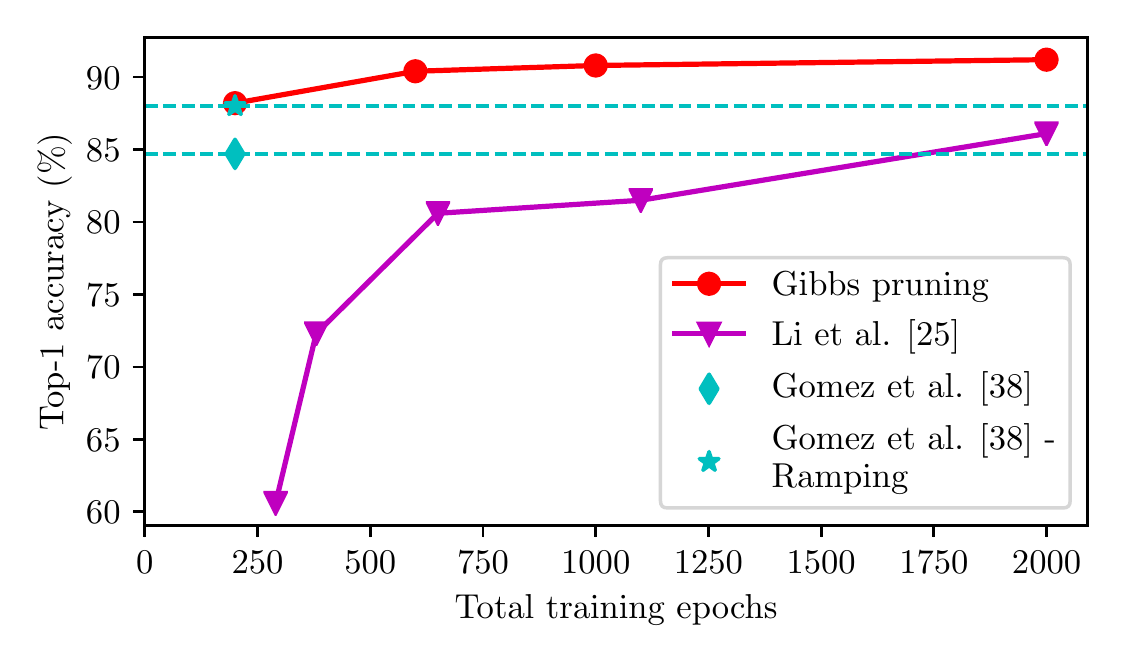}
	\label{fig:resnet56kernelscatter}
	\vspace{-7mm}
	\caption{ResNet-56}
	\end{subfigure}
	\caption[Comparison of Gibbs pruning to other kernel-wise structured pruning methods.]{Comparison of Gibbs pruning to other kernel-wise structured pruning methods. Dotted lines are shown for comparison on methods that do not require additional training epochs. }
	\label{fig:resnetkernelscatter}
\end{figure}

\begin{figure}[t!]
	\centering
	\begin{subfigure}{0.48\textwidth}
	\includegraphics[width=\textwidth]{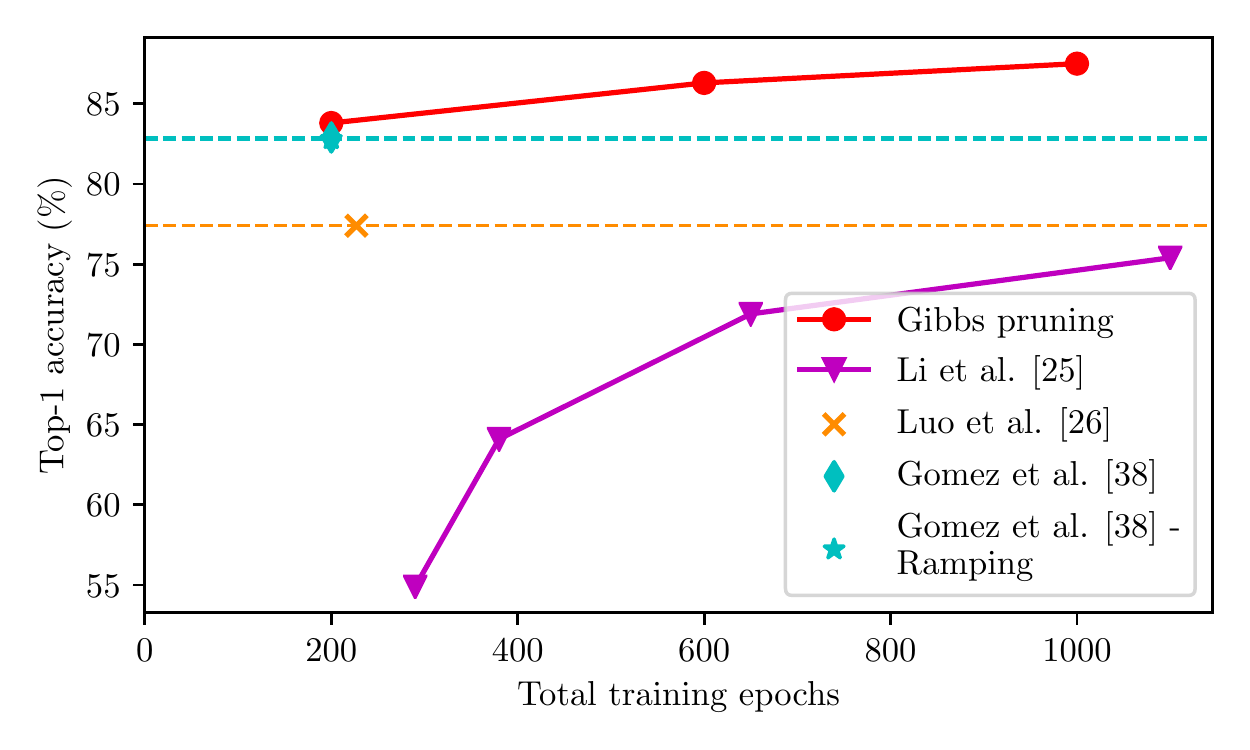}
	\label{fig:resnet20filterscatter}
	\vspace{-7mm}
	\caption{ResNet-20}
	\end{subfigure}

	\begin{subfigure}{0.48\textwidth}
	\includegraphics[width=\textwidth]{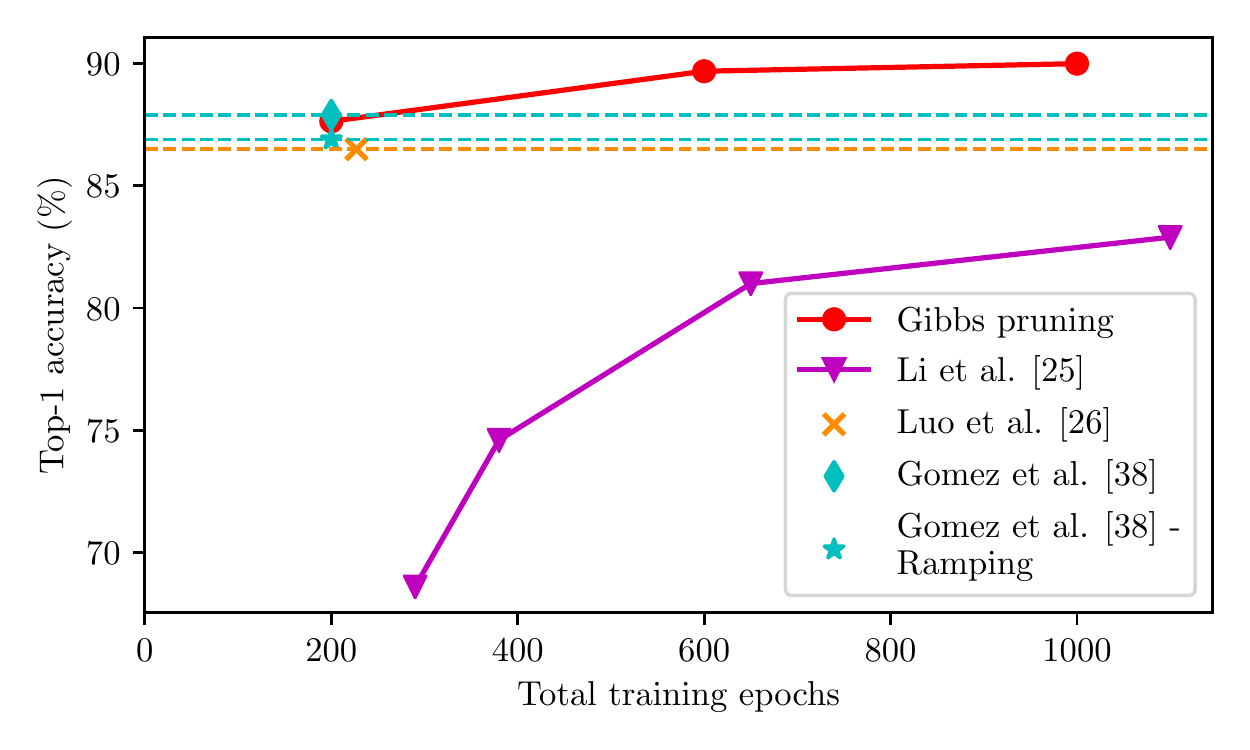}
	\label{fig:resnet56filterscatter}
	\vspace{-7mm}
	\caption{ResNet-56}
	\end{subfigure}
	\caption[Comparison of Gibbs pruning to other filter-wise structured pruning methods.]{Comparison of Gibbs pruning to other filter-wise structured pruning methods. Dotted lines are shown for comparison on methods that do not require additional training epochs. }
	\label{fig:resnetfilterscatter}
\end{figure}

We compare Gibbs pruning to three established structured pruning methods. One is targeted dropout~\cite{gomez2019learning}, implemented as previously described, but dropping entire kernels or filters. For filter-wise pruning, we set its hyperparameters to the most effective values described for unit dropout: \(\alpha=0.9, \gamma=0.75\). We also test the ramping targeted dropout variant as previously described.

Another method we evaluate is iteratively pruning filters or kernels based on their \(l_1\)-norm and retraining, as proposed by Li et al. in \cite{li2017pruning}. We use the same iterative schedules of pruning and fine-tuning as previously described for \cite{han2015learning} to test different total training lengths.

Finally, we evaluate ThiNet pruning, as proposed by Luo et al.~\cite{luo2017thinet}. This method greedily prunes channels to minimize changes in activations over an evaluation dataset. Pruning a channel is equivalent to pruning a filter in the previous layer, making channel-wise and filter-wise pruning comparable. ThiNet does not prune the \(1\times1\) convolution layers used for linear projection, which we also omit for filter-wise Gibbs pruning for comparison.

Kernel-wise pruning results are shown in Figure~\ref{fig:resnetkernelscatter} and filter-wise pruning results are shown in Figure~\ref{fig:resnetfilterscatter}. Note that a lower pruning rate of 75\% is used in evaluating filter-wise pruning methods, since filter-wise pruning at high rates results in much lower accuracy. In both types of structured pruning, Gibbs pruning outperforms other methods, and stretching the learning rate and annealing schedules proves to be effective in additionally improving accuracy.

\subsection{Analysis}
\label{sec:ev-summary}

After evaluating our proposed methods, we find that we obtain the best results with the Hamiltonians given by (\ref{eq:a2}) for unstructured pruning and (\ref{eq:struct_ham}) for structured pruning. We compare to a number of existing pruning methods and find that Gibbs pruning outperforms them. In particular, we establish a new state-of-the-art result for unstructured pruning on CIFAR-10 with ResNet-56. We also demonstrate that, unlike the methods evaluated in \cite{liu2019rethinking}, Gibbs pruning outperforms a network trained with the same pruning mask but with random weights, showing that it effectively uses network weight values in selecting pruning masks.

The following reasons could explain why our proposed method outperforms existing methods. First of all, many previous works did not test their proposed methods on sufficiently challenging tasks, either using neural networks with large dense layers or using relatively low pruning rates, meaning that their actual efficacy compared to other state-of-the-art methods was unclear. Most existing methods also do not co-adapt weights and pruning masks during training, meaning that they do not find architectures that are particularly well-adapted to the parameter values in the network, as shown in~\cite{liu2019rethinking}. Finally, our novel uses of stochastic regularization for pruning and schedule stretching are significant changes from established approaches, and they have proven especially effective.

\section{Conclusion}
\label{chap:conclusion}

We introduce a novel family of methods for neural network pruning based on Gibbs distributions that achieves high pruning rates with little reduction in accuracy. It can be used for either unstructured or structured pruning, and the general framework can be adapted to a wide range of pruning structures. We compare it to various established pruning methods and find that it outperforms them. The success of our proposed methods shows the efficacy of simultaneously training and pruning a network and using stochasticity to promote robustness under pruning.

Future work could explore other types of structured pruning that can be expressed within this framework, such as limiting the number of non-zero values in convolutional kernels or structured pruning of recurrent neural network layers. Results could also be refined by determining an optimal pruning rate for each layer given a desired global pruning rate, rather than just using the same rate for each layer. Gibbs pruning could also be combined with other network compression methods like quantization to obtain highly efficient neural networks in practice.

	\bibliographystyle{IEEEbib}
    \bibliography{post-globecom}

\begin{thebibliography}{10}

\bibitem{ackley1985learning}
David~H Ackley, Geoffrey~E Hinton, and Terrence~J Sejnowski,
\newblock ``A learning algorithm for {Boltzmann} machines,''
\newblock {\em Cognitive science}, vol. 9, no. 1, pp. 147--169, 1985.

\bibitem{geman1984stochastic}
Stuart Geman and Donald Geman,
\newblock ``Stochastic relaxation, {Gibbs} distributions, and the {Bayesian}
  restoration of images,''
\newblock {\em IEEE Transactions on Pattern Analysis and Machine Intelligence},
  , no. 6, pp. 721--741, 1984.

\bibitem{kirkpatrick1983optimization}
Scott Kirkpatrick, C~Daniel Gelatt, and Mario~P Vecchi,
\newblock ``Optimization by simulated annealing,''
\newblock {\em Science}, vol. 220, no. 4598, pp. 671--680, 1983.

\bibitem{hinton2012improving}
Geoffrey~E Hinton, Nitish Srivastava, Alex Krizhevsky, Ilya Sutskever, and
  Ruslan~R Salakhutdinov,
\newblock ``Improving neural networks by preventing co-adaptation of feature
  detectors,''
\newblock {\em arXiv preprint arXiv:1207.0580}, 2012.

\bibitem{labach2019survey}
Alex Labach, Hojjat Salehinejad, and Shahrokh Valaee,
\newblock ``Survey of dropout methods for deep neural networks,''
\newblock {\em arXiv preprint arXiv:1904.13310}, 2019.

\bibitem{labach2020framework}
Alex Labach and Shahrokh Valaee,
\newblock ``A framework for neural network pruning using {Gibbs}
  distributions,''
\newblock {\em arXiv preprint arXiv:2006.04981}, 2020.

\bibitem{howard2017mobilenets}
Andrew~G Howard, Menglong Zhu, Bo~Chen, Dmitry Kalenichenko, Weijun Wang,
  Tobias Weyand, Marco Andreetto, and Hartwig Adam,
\newblock ``{MobileNets}: Efficient convolutional neural networks for mobile
  vision applications,''
\newblock {\em arXiv preprint arXiv:1704.04861}, 2017.

\bibitem{iandola2016squeezenet}
Forrest~N Iandola, Song Han, Matthew~W Moskewicz, Khalid Ashraf, William~J
  Dally, and Kurt Keutzer,
\newblock ``{SqueezeNet: AlexNet-level} accuracy with 50x fewer parameters and<
  0.5 mb model size,''
\newblock {\em arXiv preprint arXiv:1602.07360}, 2016.

\bibitem{hubara2017quantized}
Itay Hubara, Matthieu Courbariaux, Daniel Soudry, Ran El-Yaniv, and Yoshua
  Bengio,
\newblock ``Quantized neural networks: Training neural networks with low
  precision weights and activations,''
\newblock {\em The Journal of Machine Learning Research}, vol. 18, no. 1, pp.
  6869--6898, 2017.

\bibitem{han2016deep}
Song Han, Huizi Mao, and William~J. Dally,
\newblock ``Deep compression: Compressing deep neural networks with pruning,
  trained quantization and {H}uffman coding,''
\newblock in {\em Proceedings of the International Conference on Learning
  Representations (ICLR)}, 2016.

\bibitem{denton2014exploiting}
Emily~L Denton, Wojciech Zaremba, Joan Bruna, Yann LeCun, and Rob Fergus,
\newblock ``Exploiting linear structure within convolutional networks for
  efficient evaluation,''
\newblock in {\em Advances in Neural Information Processing Systems}, 2014, pp.
  1269--1277.

\bibitem{denil2013predicting}
Misha Denil, Babak Shakibi, Laurent Dinh, Marc'Aurelio Ranzato, and Nando
  De~Freitas,
\newblock ``Predicting parameters in deep learning,''
\newblock in {\em Advances in Neural Information Processing Systems}, 2013, pp.
  2148--2156.

\bibitem{hinton2015distilling}
Geoffrey Hinton, Oriol Vinyals, and Jeffrey Dean,
\newblock ``Distilling the knowledge in a neural network,''
\newblock in {\em NIPS Deep Learning and Representation Learning Workshop},
  2015.

\bibitem{romero2014fitnets}
Adriana Romero, Nicolas Ballas, Samira~Ebrahimi Kahou, Antoine Chassang, Carlo
  Gatta, and Yoshua Bengio,
\newblock ``{FitNets}: Hints for thin deep nets,''
\newblock in {\em Proceedings of the International Conference on Learning
  Representations (ICLR)}, 2015.

\bibitem{sietsma1988neural}
{Sietsma, J.} and {Dow, R.J.F.},
\newblock ``Neural net pruning - why and how,''
\newblock in {\em IEEE 1988 International Conference on Neural Networks}, 1988,
  pp. 325--333 vol.1.

\bibitem{frankle2019stabilizing}
Jonathan Frankle, Karolina Dziugaite, Daniel~M. Roy, and Michael Carbin,
\newblock ``Stabilizing the lottery ticket hypothesis,''
\newblock {\em arXiv preprint arXiv:1903.01611}, 2019.

\bibitem{frankle2019lottery}
Jonathan Frankle and Michael Carbin,
\newblock ``The lottery ticket hypothesis: Finding sparse, trainable neural
  networks,''
\newblock in {\em Proceedings of the International Conference on Learning
  Representations (ICLR)}, 2019.

\bibitem{zhu2017prune}
Michael Zhu and Suyog Gupta,
\newblock ``To prune, or not to prune: exploring the efficacy of pruning for
  model compression,''
\newblock {\em arXiv preprint arXiv:1710.01878}, 2017.

\bibitem{guo2016dynamic}
Yiwen Guo, Anbang Yao, and Yurong Chen,
\newblock ``Dynamic network surgery for efficient dnns,''
\newblock in {\em Advances in Neural Information Processing Systems}, 2016, pp.
  1379--1387.

\bibitem{han2015learning}
Song Han, Jeff Pool, John Tran, and William Dally,
\newblock ``Learning both weights and connections for efficient neural
  networks,''
\newblock in {\em Advances in Neural Information Processing Systems}, 2015, pp.
  1135--1143.

\bibitem{lecun1990optimal}
Yann LeCun, John~S Denker, and Sara~A Solla,
\newblock ``Optimal brain damage,''
\newblock in {\em Advances in Neural Information Processing Systems}, 1990, pp.
  598--605.

\bibitem{dong2017learning}
Xin Dong, Shangyu Chen, and Sinno Pan,
\newblock ``Learning to prune deep neural networks via layer-wise optimal brain
  surgeon,''
\newblock in {\em Advances in Neural Information Processing Systems}, 2017, pp.
  4857--4867.

\bibitem{molchanov2017variational}
Dmitry Molchanov, Arsenii Ashukha, and Dmitry Vetrov,
\newblock ``Variational dropout sparsifies deep neural networks,''
\newblock in {\em Proceedings of the 34th International Conference on Machine
  Learning-Volume 70}. JMLR.org, 2017, pp. 2498--2507.

\bibitem{louizos2018learning}
Christos Louizos, Max Welling, and Diederik~P. Kingma,
\newblock ``Learning sparse neural networks through \(l_0\) regularization,''
\newblock in {\em International Conference on Learning Representations}, 2018.

\bibitem{li2017pruning}
Hao Li, Asim Kadav, Igor Durdanovic, Hanan Samet, and Hans~Peter Graf,
\newblock ``Pruning filters for efficient convnets,''
\newblock in {\em Proceedings of the International Conference on Learning
  Representations (ICLR)}, 2017.

\bibitem{luo2017thinet}
Jian-Hao Luo, Jianxin Wu, and Weiyao Lin,
\newblock ``{ThiNet}: A filter level pruning method for deep neural network
  compression,''
\newblock in {\em Proceedings of the IEEE international conference on computer
  vision}, 2017, pp. 5058--5066.

\bibitem{molchanov2016pruning}
Pavlo Molchanov, Stephen Tyree, Tero Karras, Timo Aila, and Jan Kautz,
\newblock ``Pruning convolutional neural networks for resource efficient
  inference,''
\newblock {\em arXiv preprint arXiv:1611.06440}, 2016.

\bibitem{he2018soft}
Yang He, Guoliang Kang, Xuanyi Dong, Yanwei Fu, and Yi~Yang,
\newblock ``Soft filter pruning for accelerating deep convolutional neural
  networks,''
\newblock in {\em International Joint Conference on Artificial Intelligence
  (IJCAI)}, 2018, pp. 2234--2240.

\bibitem{he2019filter}
Yang He, Ping Liu, Ziwei Wang, Zhilan Hu, and Yi~Yang,
\newblock ``Filter pruning via geometric median for deep convolutional neural
  networks acceleration,''
\newblock in {\em Proceedings of the IEEE Conference on Computer Vision and
  Pattern Recognition}, 2019, pp. 4340--4349.

\bibitem{he2017channel}
Yihui He, Xiangyu Zhang, and Jian Sun,
\newblock ``Channel pruning for accelerating very deep neural networks,''
\newblock in {\em Proceedings of the IEEE International Conference on Computer
  Vision}, 2017, pp. 1389--1397.

\bibitem{hu2016network}
Hengyuan Hu, Rui Peng, Yu-Wing Tai, and Chi-Keung Tang,
\newblock ``Network trimming: A data-driven neuron pruning approach towards
  efficient deep architectures,''
\newblock {\em arXiv preprint arXiv:1607.03250}, 2016.

\bibitem{wen2016learning}
Wei Wen, Chunpeng Wu, Yandan Wang, Yiran Chen, and Hai Li,
\newblock ``Learning structured sparsity in deep neural networks,''
\newblock in {\em Advances in Neural Information Processing Systems}, 2016, pp.
  2074--2082.

\bibitem{liu2017learning}
Zhuang Liu, Jianguo Li, Zhiqiang Shen, Gao Huang, Shoumeng Yan, and Changshui
  Zhang,
\newblock ``Learning efficient convolutional networks through network
  slimming,''
\newblock in {\em Proceedings of the IEEE International Conference on Computer
  Vision}, 2017, pp. 2736--2744.

\bibitem{ye2018rethinking}
Jianbo Ye, Xin Lu, Zhe Lin, and James~Z. Wang,
\newblock ``Rethinking the smaller-norm-less-informative assumption in channel
  pruning of convolution layers,''
\newblock in {\em International Conference on Learning Representations}, 2018.

\bibitem{yu2018nisp}
Ruichi Yu, Ang Li, Chun-Fu Chen, Jui-Hsin Lai, Vlad~I Morariu, Xintong Han,
  Mingfei Gao, Ching-Yung Lin, and Larry~S Davis,
\newblock ``{NISP:} pruning networks using neuron importance score
  propagation,''
\newblock in {\em Proceedings of the IEEE Conference on Computer Vision and
  Pattern Recognition}, 2018, pp. 9194--9203.

\bibitem{zhuang2018discrimination}
Zhuangwei Zhuang, Mingkui Tan, Bohan Zhuang, Jing Liu, Yong Guo, Qingyao Wu,
  Junzhou Huang, and Jinhui Zhu,
\newblock ``Discrimination-aware channel pruning for deep neural networks,''
\newblock in {\em Advances in Neural Information Processing Systems}, 2018, pp.
  875--886.

\bibitem{hacene2019attention}
Ghouthi~Boukli Hacene, Carlos Lassance, Vincent Gripon, Matthieu Courbariaux,
  and Yoshua Bengio,
\newblock ``Attention based pruning for shift networks,''
\newblock {\em arXiv preprint arXiv:1905.12300}, 2019.

\bibitem{gomez2019learning}
Aidan~N Gomez, Ivan Zhang, Kevin Swersky, Yarin Gal, and Geoffrey~E Hinton,
\newblock ``Learning sparse networks using targeted dropout,''
\newblock {\em arXiv preprint arXiv:1905.13678}, 2019.

\bibitem{he2018amc}
Yihui He, Ji~Lin, Zhijian Liu, Hanrui Wang, Li-Jia Li, and Song Han,
\newblock ``{AMC: AutoML} for model compression and acceleration on mobile
  devices,''
\newblock in {\em Proceedings of the European Conference on Computer Vision
  (ECCV)}, 2018, pp. 784--800.

\bibitem{gale2019state}
Trevor Gale, Erich Elsen, and Sara Hooker,
\newblock ``The state of sparsity in deep neural networks,''
\newblock {\em arXiv preprint arXiv:1902.09574}, 2019.

\bibitem{liu2019rethinking}
Zhuang Liu, Mingjie Sun, Tinghui Zhou, Gao Huang, and Trevor Darrell,
\newblock ``Rethinking the value of network pruning,''
\newblock in {\em Proceedings of the International Conference on Learning
  Representations (ICLR)}, 2019.

\bibitem{wan2013regularization}
Li~Wan, Matthew Zeiler, Sixin Zhang, Yann Le~Cun, and Rob Fergus,
\newblock ``Regularization of neural networks using dropconnect,''
\newblock in {\em International Conference on Machine Learning}, 2013, pp.
  1058--1066.

\bibitem{cipra1987introduction}
Barry~A Cipra,
\newblock ``An introduction to the {Ising} model,''
\newblock {\em The American Mathematical Monthly}, vol. 94, no. 10, pp.
  937--959, 1987.

\bibitem{wolff1989collective}
Ulli Wolff,
\newblock ``Collective {Monte Carlo} updating for spin systems,''
\newblock {\em Physical Review Letters}, vol. 62, no. 4, pp. 361, 1989.

\bibitem{gonzalez2011parallel}
Joseph Gonzalez, Yucheng Low, Arthur Gretton, and Carlos Guestrin,
\newblock ``Parallel {Gibbs} sampling: From colored fields to thin junction
  trees,''
\newblock in {\em Proceedings of the Fourteenth International Conference on
  Artificial Intelligence and Statistics}, 2011, pp. 324--332.

\bibitem{brooks2011handbook}
Steve Brooks, Andrew Gelman, Galin Jones, and Xiao-Li Meng,
\newblock {\em Handbook of {Markov chain Monte Carlo}},
\newblock CRC press, 2011.

\bibitem{blalock2020state}
Davis Blalock, Jose Javier~Gonzalez Ortiz, Jonathan Frankle, and John Guttag,
\newblock ``What is the state of neural network pruning?,''
\newblock {\em arXiv preprint arXiv:2003.03033}, 2020.

\bibitem{he2016deep}
Kaiming He, Xiangyu Zhang, Shaoqing Ren, and Jian Sun,
\newblock ``Deep residual learning for image recognition,''
\newblock in {\em Proceedings of the IEEE conference on computer vision and
  pattern recognition}, 2016, pp. 770--778.

\bibitem{krizhevsky2012imagenet}
Alex Krizhevsky, Ilya Sutskever, and Geoffrey~E Hinton,
\newblock ``{ImageNet} classification with deep convolutional neural
  networks,''
\newblock in {\em Advances in Neural Information Processing Systems}, 2012, pp.
  1097--1105.

\bibitem{krizhevsky2012cifar}
Alex Krizhevsky,
\newblock ``Learning multiple layers of features from tiny images,''
\newblock {\em University of Toronto}, 05 2012.

\bibitem{kingma2014adam}
Diederik Kingma and Jimmy Ba,
\newblock ``Adam: A method for stochastic optimization,''
\newblock {\em arXiv preprint arXiv:1412.6980}, 2014.

\bibitem{goodfellow2016deep}
Ian Goodfellow, Yoshua Bengio, and Aaron Courville,
\newblock {\em Deep Learning},
\newblock MIT Press, 2016,
\newblock \url{http://www.deeplearningbook.org}.

\end{thebibliography}
\end{document}